\documentclass[10pt]{article}
\usepackage{float}
\usepackage[
    a4paper,
    left=1.3in,
    right=1.3in,
    top=1.15in,
    bottom=1.15in
]{geometry}
\usepackage[T1]{fontenc}
\usepackage{lmodern}
\usepackage{microtype}
\usepackage{natbib}
\usepackage{amsmath,amssymb,mathtools,amsthm}

\usepackage{graphicx}
\usepackage{subcaption}
\usepackage{booktabs}
\usepackage{tabularx}
\usepackage{makecell}

\usepackage{pifont}
\usepackage{xcolor}
\usepackage{xfrac}

\usepackage{setspace}

\usepackage{algorithm}
\usepackage{algorithmic}

\usepackage{hyperref}
\hypersetup{
    colorlinks=true,
    linkcolor=blue,
    citecolor=blue,
    urlcolor=blue
}

\usepackage[textsize=tiny]{todonotes}

\usepackage{authblk}

\usepackage{amsthm}

\theoremstyle{plain}
\newtheorem{theorem}{Theorem}[section]

\newtheorem{lemma}[theorem]{Lemma}

\theoremstyle{definition}

\newtheorem{assumption}[theorem]{Assumption}
\theoremstyle{remark}
\newtheorem{remark}[theorem]{Remark}

\usepackage[capitalize,noabbrev]{cleveref}

\def\reals{\mathbb{R}}
\newcommand{\real}{\mathbb{R}}

\newcommand{\supp}{\textbf{Support}\{\D\}}
\newcommand{\sigmal}{\sigma_L}

\newcommand{\D}{\mathcal{D}}

\newcommand{\K}{\mathcal{K}}
\newcommand{\T}{\mathcal{T}}

\newcommand{\PS}{\mathcal{PS}}

\newcommand{\ba}{\mathbf{a}}
\newcommand{\bb}{\mathbf{b}}
\newcommand{\bc}{\mathbf{c}}
\newcommand{\bd}{\mathbf{d}}

\newcommand{\bg}{\mathbf{g}}
\newcommand{\tbg}{\Tilde{\bg}}
\newcommand{\bbg}{\Bar{\bg}}

\newcommand{\bbm}{\mathbf{m}}
\newcommand{\hbm}{\hat{\bbm}}

\newcommand{\bq}{\mathbf{q}}
\newcommand{\hbq}{\hat{\bq}}

\newcommand{\bs}{\mathbf{s}}
\newcommand{\bbs}{\Bar{\bs}}

\newcommand{\bw}{\mathbf{w}}
\newcommand{\bx}{\mathbf{x}}
\newcommand{\bxs}{\bx^\star}
\newcommand{\by}{\mathbf{y}}
\newcommand{\bz}{\mathbf{z}}
\newcommand{\beps}{\boldsymbol{\epsilon}}
\newcommand{\hbeps}{\hat{\beps}}
\newcommand{\bbeps}{\Bar{\beps}}

\newcommand{\Exp}[1]{\mathbb{E}#1} 
\newcommand{\ExpB}[1]{\mathbb{E}\left[#1\right]} 
\newcommand{\ExpC}[2]{\mathbb{E}_{#1}\left[#2\right]} 
\newcommand{\norm}[1]{\left\|#1\right\|} 
\newcommand{\normsq}[1]{\left\|#1\right\|^2} 
\newcommand{\dotprod}[2]{\left\langle#1,#2\right\rangle} 
\newcommand{\f}[1]{f\!\left(#1\right)} 
\newcommand{\df}[1]{\nabla f\!\left(#1\right)} 
\newcommand{\proj}[1]{\Pi_\K\left(#1\right)} 
\newcommand{\prob}[1]{\Pr\left\{#1\right\}} 
\newcommand{\oo}[1]{\mathcal{O}\left(#1\right)} 
\newcommand{\ot}[1]{\Theta\left(#1\right)} 
\newcommand{\mini}[1]{\min\left\{#1\right\}} 
\def\al#1\eal{\begin{align}#1\end{align}} 
\def\als#1\eals{\begin{align*}#1\end{align*}} 
\newcommand{\non}{\nonumber\\} 

\title{Bringing Order to Asynchronous SGD: Towards Optimality under Data-Dependent Delays with Momentum}

\author[1]{Tehila Dahan$^*$}
\author[1]{Roie Reshef$^*$}
\author[1]{Sharon Goldstein}
\author[1]{Kfir Y. Levy}

\affil[1]{Technion -- Israel Institute of Technology, Haifa, Israel}
\date{}

\begin{document}
\maketitle
\begingroup
\renewcommand\thefootnote{}
\footnotetext{$^*$Equal contribution.}
\endgroup
\begin{abstract}
Asynchronous SGD enables scalable distributed training but suffers from stale gradients.
When delays depend on the data, slower samples may be underrepresented, biasing training toward faster samples.
We introduce \emph{ordered momentum}, a unified framework that attains the best-known rates for smooth convex and non-convex objectives under both data-independent and data-dependent delays.
Notably, we establish (i) the first convergence guarantee for smooth convex objectives with data-dependent delays and, among analyses of data-dependent delays, the first to (ii) benefit from parallelization and (iii) match the tight data-independent rate, with a leading stochastic term independent of the number of workers.
Finally, we derive robust learning rates that simplify hyperparameter tuning across convex and non-convex settings.
\end{abstract}

\section{Introduction}

The growing scale of modern machine-learning models and datasets has made distributed training indispensable \citep{verbraeken2020survey,zhao2023survey}.
Among the available approaches, \textit{asynchronous parallelization} is particularly attractive: workers compute gradients independently and update a shared model without waiting for global synchronization \citep{mishchenko2022asynchronous}.
By eliminating synchronization barriers, asynchronous SGD reduces idle time due to stragglers, accommodates heterogeneous hardware, and can substantially improve wall-clock throughput compared to its synchronous counterpart \citep{dekel2012optimal}.
This efficiency, however, comes at the cost of \textit{staleness}: while a worker computes a gradient, other workers may update the shared model; consequently, by the time the gradient is applied, it may correspond to an earlier (and potentially substantially different) model iterate.

A key assumption underlying tight convergence guarantees for asynchronous optimization is that delays are independent of the data \citep{cohen2021asynchronous,attia2024faster,aviv2021asynchronous,mishchenko2022asynchronous,koloskova2022sharper,shi2024ordered}.
In modern training pipelines, however, computation time often depends directly on the input.
For example, language-modeling sequences are commonly grouped by length to improve hardware utilization \citep{guo2025deepseek,pouransari2024dataset}.
Because longer sequences require more computation, their gradients arrive systematically later and are therefore more stale.
The same phenomenon arises in video understanding, where longer clips require more computation.
Delays may therefore be coupled with the data distribution, causing faster examples to influence the model earlier and more frequently than slower ones.
This implicit sampling bias can destabilize training, slow convergence, and degrade generalization.

This distinction determines how staleness can be handled.
When delays are independent of the data, stale gradients can be attenuated without systematically altering the effective data distribution.
This observation motivates optimal methods that either adapt the learning rate to the observed delay \citep{mishchenko2022asynchronous,koloskova2022sharper,shi2024ordered,venigalla2020adaptive,barkai2019gap,wu2022delay} or discard updates beyond a prescribed delay threshold \citep{cohen2021asynchronous,attia2024faster,maranjyan2025ringmaster}.
When delays depend on the data, however, staleness carries information about the underlying sample: gradients from computationally expensive examples tend to be more stale.
The same mechanisms then systematically downweight or discard these examples, implicitly reweighting the data distribution toward faster samples.
Avoiding this bias requires a different approach.
Indeed, the best-known guarantees for data-dependent delays generally rely on a fixed, delay-independent learning rate \citep{mishchenko2022asynchronous}.
Consequently, attaining the best-known rates in the two delay regimes typically requires different algorithms.

In this work, we introduce the \emph{ordered momentum} framework, a unified approach that achieves the best-known convergence rates for convex and non-convex objectives under both data-independent and data-dependent delays.
For smooth convex objectives, we (i) establish the first convergence guarantee under data-dependent delays and, among existing analyses of data-dependent delays, are the first to (ii) exhibit a benefit from parallelization and (iii) match the tight data-independent rate, with a leading stochastic term independent of the number of workers $M$.

To see this gap, the best-known convex rate under data-dependent delays is $\oo{\frac{DG\sqrt{M}}{\sqrt{T}}}$ \citep{mishchenko2022asynchronous}, where $D$ is the domain diameter and $G$ bounds the stochastic gradient norm.
This guarantee does not exploit smoothness, and its leading term deteriorates as $\sqrt{M}$.
A similar dependence appears in the best-known non-convex rate, whose leading term scales as $\oo{(\frac{GL\Delta M}{T})^{2/3}}$, where $\Delta$ is the initial optimality gap.

By contrast, for $L$-smooth objectives under data-independent delays, the tight rate in the convex setting is $\oo{\frac{D^2LM}{T}+\frac{D\sigma}{\sqrt{T}}}$, and in the non-convex setting it is $\oo{\frac{L\Delta M}{T}+\frac{\sigma\sqrt{L\Delta}}{\sqrt{T}}}$ \citep{mishchenko2022asynchronous}, where $\sigma\in[0,G]$ bounds the stochastic gradient variance (see Appendix~\ref{app:tech}).
Crucially, in both rates, $M$ appears only in a lower-order term, while the dominant stochastic term is independent of the number of workers.
Existing guarantees under data-dependent delays therefore fail to capture this parallelization benefit (see Remark~\ref{remark:convex-optimality}).
Our convex analysis closes this gap by recovering an $M$-independent dominant term even when delays depend on the data, matching the parallelization benefit of the data-independent regime.

\paragraph{Our Contributions.}
In this work, we propose a momentum-based framework for asynchronous optimization.
Our main contributions are as follows:
\begin{itemize}
\item
We introduce a unified approach attaining the best-known rates for smooth convex and non-convex objectives under both data-independent and data-dependent delays.
For convex objectives, we establish the first data-dependent guarantee matching the tight data-independent rate.
Unlike the previous best bound, $\oo{\frac{DG\sqrt{M}}{\sqrt{T}}}$, whose leading term deteriorates with the number of workers $M$, our smooth convex guarantee has an $M$-independent leading term, capturing a parallelization benefit for the first time under data-dependent delays.
\item
Our guarantees use a fixed learning rate that can be tuned once offline.
In contrast, common delay-adaptive methods \citep{mishchenko2022asynchronous,koloskova2022sharper,shi2024ordered,venigalla2020adaptive,barkai2019gap,wu2022delay}
use iteration-dependent schedules of the form $\eta_t=\mini{\oo{\frac{1}{L\tau_t}},\oo{\frac{1}{LM}},\oo{\sqrt{\frac{\Delta}{L\sigma^2T}}}}$, potentially suggesting costly online tuning as the delay $\tau_t$ varies.
Our framework attains the best-known guarantees while retaining standard offline tuning.
Moreover, by leveraging a double-momentum mechanism, our framework further preserves the optimal rate across a broad range of learning rates, reducing tuning sensitivity in highly variable asynchronous environments.
\end{itemize}

\paragraph{Related Work.}
Asynchronous SGD has long been studied as a way to scale optimization by eliminating synchronization barriers \citep{tsitsiklis2003distributed,mishchenko2022asynchronous}.
Early analyses primarily considered \emph{fixed delays} \citep{arjevani2020tight,stich2020error,zheng2017asynchronous,giladi2019stability}, an abstraction suited to homogeneous systems but unable to capture the variability of heterogeneous clusters and edge devices.
Subsequent works introduced \emph{virtual iterates}, initially under bounded or fixed delays \citep{mania2017perturbed,stich2020error} and later in sharper analyses \citep{mishchenko2022asynchronous,koloskova2022sharper,shi2024ordered}.
This approach compares the asynchronous iterates with an idealized delay-free trajectory, enabling tighter convergence guarantees.
Despite this progress, no existing analysis exploits smoothness to obtain tight convex guarantees under data-dependent delays.

Momentum methods are central to modern machine learning optimization \citep{polyak1964some,nesterov2013introductory,diederik2014adam,loshchilov2017decoupled}.
Early asynchronous momentum methods were analyzed primarily under restrictive assumptions, such as bounded \citep{barkai2019gap,hakimi2019taming,yanmomentum} or fixed delays \citep{giladi2019stability}, limiting their applicability to highly variable environments.
More recently, \citet{shi2024ordered} introduced \emph{ideal momentum}, which assigns each arriving gradient its corresponding synchronous weight, while grouping updates into batches of size $M$ with uniform within-group weights.
However, their data-independent guarantees require delay-adaptive learning rates and potentially costly online tuning.
In contrast, we introduce refined notions of \emph{ideal momentum} that yield tight guarantees under both data-independent and data-dependent delays using a fixed learning rate that can be tuned once offline.
Moreover, by incorporating double momentum \citep{dahanstochastic}, we further establish the first optimal guarantee for smooth convex optimization under data-dependent delays, matching the tight rate known for data-independent delays.
Finally, we show that this optimal rate holds across a broad range of learning rates, demonstrating robustness even in highly variable asynchronous environments.

\section{Settings}
We consider minimizing the smooth objective
\als
\f{\bx}:=\ExpC{\bz\sim\D}{\f{\bx;\bz}}
\eals
where $\D$ is an unknown distribution and samples are drawn i.i.d. We study iterative first-order methods in a distributed system with $M$ workers and a central parameter server ($\PS$).
Given $\bx$, each worker computes a stochastic gradient $\bg:=\df{\bx;\bz}$, where $\bz\sim\D$, satisfying $\ExpB{\bg\mid\bx}=\df{\bx}$.

We focus on the \emph{asynchronous} setting, where the $\PS$ updates the global model immediately upon receiving a worker update.
Consequently, gradients may be computed at stale iterates and applied out of order.
We denote by $\tau_t\geq0$ the delay of the update applied at iteration $t$.

\paragraph{Asynchronous SGD.}
In this asynchronous setting, the standard SGD update \citep{mishchenko2022asynchronous, koloskova2022sharper} takes the form
\al\label{eq:async-sgd}
\bx_{t+1}=\bx_t-\eta_t\bg_{t-\tau_t}
\eal
where $\eta_t>0$ is the learning rate and $\bg_{t-\tau_t}:=\df{\bx_{t-\tau_t};\bz_{t-\tau_t}}$ denotes a stochastic gradient evaluated at a delayed iterate $\bx_{t-\tau_t}$.

\paragraph{Assumptions.}
Throughout the paper, we consider four regimes that combine nonconvex or convex objectives with data-independent or data-dependent delays.
Each regime uses the corresponding assumptions stated below.
Stochastic statements are with respect to $\bz\sim\D$, and the domain is $\reals^d$ in the nonconvex setting and $\K$ in the convex setting.

\begin{assumption}[Smooth nonconvex objective]\label{assume:non-conv}
The objective $f:\real^d\to\reals$ is lower bounded by $f^\star:=\inf_{\bx\in\real^d}\f{\bx}>-\infty$ and has an $L$-Lipschitz gradient:
\als
\norm{\df{\bx}-\df{\by}}\leq L\norm{\bx-\by}, \qquad \forall\bx,\by\in\reals^d
\eals
We write $\Delta:=\f{\bx_1}-f^\star$.
\end{assumption}

\begin{assumption}[Smooth convex objective]\label{assume:bound-conv}
The set $\K\subseteq\real^d$ is nonempty, compact, and convex, with diameter $D$, and $f:\K\to\real$ is convex.
Moreover, every sample loss has an $L$-Lipschitz gradient:
\als
\norm{\df{\bx;\bz}-\df{\by;\bz}}\leq L\norm{\bx-\by}, \qquad \forall\bx,\by\in\K,\ \bz\in\supp
\eals
This implies (see Appendix~\ref{app:tech}) that $f$ is $L$-smooth and that, for some $\sigmal\in[0,L]$:
\als
\Exp{\normsq{\left(\df{\bx;\bz}-\df{\bx}\right)-\left(\df{\by;\bz}-\df{\by}\right)}}\leq\sigmal^2\normsq{\bx-\by}, \qquad \forall\bx,\by\in\K
\eals
We denote $\bxs:=\arg\min_{x\in\K}\f{\bx}$, $f^\star:=\f{\bxs}$, and $G^\star:=\norm{\df{\bxs}}$.
\end{assumption}

\begin{assumption}[Bounded stochastic noise for data-independent delays]\label{assume:indep}
The delay is independent of the stochastic samples, and there exist $\sigma\geq0$, such that $\Exp{\normsq{\df{\bx;\bz}-\df{\bx}}}\leq\sigma^2$, for every $\bx$ in the relevant domain.
\end{assumption}

\begin{assumption}[Bounded stochastic gradients for data-dependent delays]\label{assume:dep}
There exists $G>0$, such that $\norm{\df{\bx;\bz}}\leq G$, for every $\bx$ in the relevant domain and every $\bz\in\supp$.
This implies $\Exp{\normsq{\df{\bx;\bz}-\df{\bx}}}\leq\sigma^2$ for some $\sigma\in[0,G]$ (see Appendix~\ref{app:tech}).
\end{assumption}

\textbf{Notations.}
$\norm{\cdot}$ denotes the Euclidean $(\ell_2)$ norm.
For any $N\in\mathbb{N}$, we define $[N]:=\{1,\ldots,N\}$.

\section{Momentum-Based Asynchronous Framework}\label{sec:non-convex}

Momentum is a cornerstone of modern optimization, accelerating convergence while stabilizing noisy stochastic updates \citep{polyak1964some,nesterov2013introductory,diederik2014adam,loshchilov2017decoupled,jordan6muon}.
At a high level, it extends vanilla SGD by maintaining a running average of past gradients, combining historical information with the current gradient to determine the update direction.
Formally, in the single-worker setting, the classical momentum method of \citet{polyak1964some} updates the momentum buffer and model parameters as follows:
\al\label{eq:single-worker-momentum}
\bbm_t^{\text{syn}}=\beta\bg_t+(1-\beta)\bbm_{t-1}^{\text{syn}}, \quad \bx_{t+1}=\bx_t-\eta\bbm_t^{\text{syn}}
\eal
where $\eta>0$ is the learning rate, $\beta\in(0,1)$ is the momentum parameter, $\bg_t:=\df{\bx_t;\bz_t}$ is the stochastic gradient at step $t$, and $\bbm_t^{\text{syn}}$ is the momentum in the single-worker synchronous baseline.

\paragraph{Momentum-Based Solution.}
A natural extension of momentum to the asynchronous setting is to replace the current stochastic gradient with a delayed one.
Specifically, let $\tau_t\geq0$ denote the delay at iteration $t$, so that $\bg_{t-\tau_t}:=\df{\bx_{t-\tau_t};\bz_{t-\tau_t}}$ is evaluated at the stale iterate $\bx_{t-\tau_t}$.
The resulting asynchronous momentum update is
\al\label{eq:naive_asy}
\bbm^{\text{asy}}_t=\beta\bg_{t-\tau_t}+(1-\beta)\bbm^{\text{asy}}_{t-1}
\eal
where $\bbm_t^{\mathrm{asy}}$ denotes the asynchronous momentum buffer.
When $\tau_t=0$ for all $t$, the update reduces to the synchronous method, and $\bbm_t^{\mathrm{asy}}=\bbm_t^{\mathrm{syn}}$.
Despite its simplicity, this extension has a fundamental drawback: momentum can assign greater weight to gradients evaluated at older iterates than to those evaluated more recently.
To see this, it is useful to unroll the synchronous momentum recursion in~\cref{eq:single-worker-momentum}, making the contribution of each past gradient explicit:
\al\label{eq:poly-momentum}
\bbm_t^{\text{syn}}=\beta\sum_{\tau=0}^{t-1}(1-\beta)^\tau\bg_{t-\tau}
\eal
This expansion shows that momentum forms an \emph{exponential moving average} of past gradients.
Thus, the gradient from $\tau$ iterations ago is weighted by exactly $\beta(1-\beta)^\tau$, where $(1-\beta)^\tau$ reflects its “distance” from the current iterate.
In other words, momentum naturally emphasizes recent gradients while retaining a diminishing contribution from older ones.

\paragraph{Our Proposal.}
The unrolled recursion in~\cref{eq:poly-momentum} provides a simple interpretation of momentum; in the ideal delay-free setting, each stochastic gradient contributes to the momentum buffer according to its exact position in an exponentially decaying sequence: a gradient computed $\tau$ iterations in the past is weighted by $\beta(1-\beta)^{\tau}$.
Preserving this ideal ordering suggests that, under delays, each gradient should be incorporated using precisely the weight it would receive in the single-worker trajectory.
Motivated by this observation, we modify the asynchronous momentum update to
\al\label{eq:ordered-momentum}
\bbm_t=\beta(1-\beta)^{\tau_t}\bg_{t-\tau_t}+(1-\beta)\bbm_{t-1}
\eal
A related “ordered weighting” perspective was recently highlighted by \citet{shi2024ordered}, who group iterations into blocks of size $M$ and assign each block a uniform weight in their \emph{virtual momentum} construction that mimics the ideal momentum.
In contrast, we adopt the same underlying principle but enforce it exactly.
Specifically, each delayed gradient $\bg_{t-\tau_t}$ is assigned the exact coefficient $\beta(1-\beta)^{\tau_t}$ it would have in the ideal, delay-free trajectory.
As a result, stale gradients are neither discarded nor treated as fresh; instead, they are consistently integrated into our momentum structure, with any deviation arising solely from missing future gradients that have not yet arrived.

Formally, for iteration $t\geq1$ and index $1\leq k\leq t$, define the \textit{weighted gradient} as
\als
\bg_{t,k}:=\beta(1-\beta)^{t-k}\bg_k
\eals
where for $k=1$, we only use the contribution of the first worker, while the other workers replace the first gradient they receive with zero.
Thus, the resulting momentum with ordered weights (\cref{eq:ordered-momentum,alg:momentum}) can be represented as
\al\label{eq:moment_real}
\bbm_t=\sum_{k\in[t]\setminus\T(t)}\bg_{t,k}
\eal
Here, $\T(t)$ denotes the set of iterations whose gradients are still pending at time $t$, and therefore $|\T(t)|\leq M-1$.
Thus, \cref{eq:moment_real} contains all gradients associated with iterations up to $t$ that have already arrived, while omitting those indexed for the iterations that have not yet arrived and are therefore still missing.
The resulting momentum buffer preserves the weighting of the ideal ordered momentum, differing only by at most $M-1$ pending gradients.
Based on this representation, we define the ideal momentum $\hbm_t$ at iteration $t$ as
\al\label{eq:moment_ideal}
\hbm_t:=\sum_{k=1}^t\bg_{t,k}
\eal
Where we add the gradients to fill in the missing iterations.
We shall \emph{use these gradients only for the sake of analysis}.
Specifically, we shall use these gradients to define an ideal (impractical) version of the momentum, which serves as an analytical tool for comparing with and analyzing the practical momentum used in \cref{alg:momentum}.
Finally, the stochastic error of delayed momentum, $\beps_t:=\bbm_t-\df{\bx_t}$, can be decomposed as
\als
\beps_t=\hbeps_t+\bb_t
\eals
where $\hbeps_t=\hbm_t-\df{\bx_t}$ is the error of the ideal momentum and $\bb_t$ is the difference induced by delay.
This difference has the explicit form
\al\label{eq:moment_delay}
\bb_t&:=\bbm_t-\hbm_t=-\sum_{k\in\T(t)}\bg_{t,k}
\eal
which scales at most linearly with the number of workers, i.e., $\norm{\bb_t}\leq\oo{M}$.
Consequently, the stochastic error of the delayed momentum, $\beps_t$, is controlled by two terms: (i) an $\oo{M}$ bias from the pending gradients and (ii) the same stochastic error that appears in the ideal delay-free setting.
The key intuition is that our asynchronous momentum aims to be \emph{as close as possible to ideal, delay-free momentum}, so its bias depends only on the at most $M-1$ missing updates, rather than on the magnitude of their delays.

\begin{algorithm}[t]
\caption{Asynchronous Momentum with Ordered Weights}
\label{alg:momentum}
\begin{algorithmic}
\STATE \textbf{Input:} \#Iterations $T$, initial point $\bx_1\in\real^d$, learning rate $\eta>0$, momentum parameter $\beta\in(0,1)$.
\STATE Set $\bbm_0=\mathbf{0}^d$.
\FOR{$t=1$ to $T$}
\STATE Receive $\bg_{t-\tau_t}:=\df{\bx_{t-\tau_t};\bz_{t-\tau_t}}$ from a worker, and replace $\bg_{t-\tau_t}:=\mathbf{0}^d$ on its first response for workers other than the first, i.e., when $t-\tau_t=1$ and $t>1$.
\STATE Update: $\bbm_t=\beta(1-\beta)^{\tau_t}\bg_{t-\tau_t}+(1-\beta)\bbm_{t-1}$
\STATE Update: $\bx_{t+1}=\bx_t-\eta\bbm_t$
\STATE Send $\bx_{t+1}$ back to the worker.
\ENDFOR
\STATE \textbf{Output:} $\Bar{\bx}_T$ (chosen uniformly at random from the set $\{\bx_1,\ldots,\bx_T\}$).
\end{algorithmic}
\end{algorithm}

\paragraph{Ordered Momentum for Data-Independent Delays.}
\begin{theorem}[Non-convex Smooth Objectives on Data-Independent Delays]\label{thm:momentum_indep}
Let $f$ satisfy assumptions~\ref{assume:non-conv} and \ref{assume:indep}.
Consider the iterates $\{\bx_t\}_{t=1}^T$ generated by~\cref{alg:momentum} with $\beta=\mini{\frac{1}{12(M-1)},\frac{4\sqrt{L\Delta}}{\sigma\sqrt{3T}}}$, and $\eta=\frac{\beta}{2\sqrt{6}L}$.
Then:
\als
\frac{1}{T}\sum_{t=1}^T\Exp{\normsq{\df{\bx_t}}}\leq\oo{\frac{L\Delta M}{T}+\frac{\sigma\sqrt{L\Delta}}{\sqrt{T}}}
\eals
\end{theorem}

\paragraph{Ordered Momentum for Data-Dependent Delays.}
\begin{theorem}[Non-convex Lipschitz-Smooth Objectives on Data-Dependent Delays]\label{thm:momentum_dep}
Let $f$ satisfy assumptions~\ref{assume:non-conv} and \ref{assume:dep}.
Consider the iterates $\{\bx_t\}_{t=1}^T$ generated by~\cref{alg:momentum} with $\beta=\mini{\frac{1}{2},\frac{\sqrt{13L\Delta}}{\sigma\sqrt{3T}},\left(\frac{13L\Delta}{6G^2(M-1)^2T}\right)^{1/3}}$, and $\eta=\frac{\beta}{2\sqrt{3}L}$.
Then:
\als
\frac{1}{T}\sum_{t=1}^T\Exp{\normsq{\df{\bx_t}}}\leq\oo{\frac{L\Delta}{T}+\frac{\sigma\sqrt{L\Delta}}{\sqrt{T}}+\left(\frac{GL\Delta M}{T}\right)^{2/3}}
\eals
\end{theorem}

\vspace{5pt}
\begin{remark}[A unified approach to both delay regimes]\label{remark:unified}
Existing optimal methods typically treat data-independent and data-dependent delays differently.
For example, in asynchronous SGD, under data-dependent delays, the optimal guarantee can be achieved with a fixed learning rate \citep{mishchenko2022asynchronous, koloskova2022sharper}, whereas attaining the optimal rate under data-independent delays typically requires adapting the learning rate to the realized delay \citep{mishchenko2022asynchronous, koloskova2022sharper}.
In contrast, our method matches the best-known guarantees for general smooth non-convex objectives in both regimes within the same algorithmic framework, using a fixed learning rate.
\end{remark}

\vspace{5pt}
\begin{remark}[A fixed, delay-independent learning rate]\label{remark:practical}
Several optimal methods for data-independent delays use learning rates that adapt to the realized delay \citep{mishchenko2022asynchronous,koloskova2022sharper,shi2024ordered}, for example:
\al\label{eq:delay-adaptive-sgd}
\eta_t=\mini{\ot{\frac{1}{L\tau_t}},\ot{\frac{1}{LM}},\ot{\sqrt{\frac{\Delta}{L\sigma^2T}}}}
\eal
Implementing such schedules can be difficult in practice: the quantities entering the minimum, such as $L$, $\Delta$, and $\sigma^2$, are generally unknown, while the realized delay $\tau_t$ may vary substantially across iterations.
Consequently, a learning rate tuned for one delay may be unsuitable for another, potentially requiring repeated online retuning rather than a single offline search.
In contrast, our method uses a fixed learning rate that is independent of the realized delays.
It can therefore be tuned once offline, for example via a standard grid search, and kept unchanged throughout training.
\end{remark}

\section{Towards Optimality via Double Momentum}

Despite substantial progress in asynchronous optimization, convergence in the smooth convex setting with data-dependent delays has remained unresolved.
We close this gap by establishing the first convergence guarantee for this setting, matching the tight rate known for data-\textbf{in}dependent delays \citep{mishchenko2022asynchronous}.
To obtain this result, we instantiate our framework with the double-momentum mechanism of $\mu^2$-SGD~\citep{dahanstochastic}, which combines two complementary forms of momentum: (i) STORM, a momentum-based variance-reduction correction that controls stochastic gradient noise \citep{cutkosky2019momentum}, and (ii) parameter-level momentum, also known as Anytime-SGD \citep{cutkosky2019anytime}.
Finally, through our ordered-momentum framework, we show that the resulting convergence guarantees remain robust across a broad range of learning rates, even in highly unstable asynchronous environments with data-dependent delays.

\textbf{Notations.} For any sequence $\{a_n\}_n$, we use the shorthand $a_{i:j}:=\sum_{n=i}^ja_n$, and for every $\bx\in\mathbb{R}^d$, we denote the orthogonal projection of $\bx$ onto a set $\K$ by $\proj{\bx}:=\arg\min_{\by\in\K}\normsq{\by-\bx}$.

\subsection{First Momentum: Stochastic Recursive Momentum}

The first momentum component in $\mu^2$-SGD is STORM \citep{cutkosky2019momentum}.
STORM extends classical momentum \citep{polyak1964some} by incorporating a sample-reuse correction term into the momentum recursion:
\als
\bd_t^\text{syn}=\underbrace{(1-\beta)\bd_{t-1}^\text{syn}+\beta\bg_t}_{\textnormal{classical momentum}}+\underbrace{(1-\beta)\left(\bg_t-\tbg_{t-1}\right)}_{\textnormal{correction term}}
\eals
where $\bg_t:=\df{\bx_t;\bz_t}$ and $\tbg_{t-1}:=\df{\bx_{t-1};\bz_t}$ use the same sample $\bz_t$.
Applying our asynchronous framework to this recursion yields the \textbf{ordered asynchronous STORM} update:
\al\label{eq:ordered-storm-fixed}
\bd_t=(1-\beta)\bd_{t-1}+(1-\beta)^{\tau_t}(\bg_{t-\tau_t}-(1-\beta)\tbg_{t-\tau_t-1})
\eal 
This recovers the same momentum term as in~\cref{eq:ordered-momentum}, along with the STORM correction term, which is likewise ordered according to the correct time index.

\paragraph{Role of the correction.}
The STORM correction leverages objective smoothness to correlate gradient estimates at consecutive iterates.
In particular, the stochastic error in the gradient difference is controlled by $\norm{\bx_t-\bx_{t-1}}$, thereby reducing the estimator variance \citep{cutkosky2019momentum}.
Classical momentum stabilizes the updates but generally introduces a persistent bias: $\ExpB{\bbm_t^{\mathrm{syn}}-\df{\bx_t}}\neq0$.
In contrast, the STORM estimator is unbiased in total expectation, $\ExpB{\bd_t^{\mathrm{syn}}-\df{\bx_t}}=0$ (but not conditionally unbiased).
Equivalently, its update can be written as
\als
\bd_t^{\mathrm{syn}}=\bg_t+\underbrace{(1-\beta)\left(\bd_{t-1}^{\mathrm{syn}}-\tbg_{t-1}\right)}_{\textnormal{unbiased term}}
\eals

\subsection{Second Momentum: Anytime-SGD}
The Anytime-GD mechanism, introduced by \citet{cutkosky2019anytime}, applies gradient updates to an auxiliary sequence while evaluating gradients at an averaged query point.
Specifically, it maintains gradient-descent iterates $\{\bw_t\}_t$ and query points $\{\bx_t\}_t$, where $\bx_t$ is a weighted average of $\{\bw_\tau\}_{\tau=1}^t$ with positive weights $\{\alpha_\tau\}_{\tau=1}^t$.
The updates are
\als
\bw_{t+1}=\proj{\bw_t-\eta\alpha_t\bg_t}, \qquad \bx_{t+1}=\bx_t+\frac{\alpha_{t+1}}{\alpha_{1:t+1}}\left(\bw_{t+1}-\bx_t\right)
\eals
where $\bg_t:=\df{\bx_t;\bz_t}$ is a stochastic gradient evaluated at the query point $\bx_t$.
A key property of Anytime-GD is that the averaging update ensures:
\als
\norm{\bx_{t+1}-\bx_t}\leq\frac{\alpha_{t+1}}{\alpha_{1:t+1}}D
\eals
Thus, parameter-level momentum brings successive query points closer together, providing additional stability, and when combined with STORM, this contraction further reduces stochastic error and improves robustness, as shown by \citet{dahanstochastic}.

In asynchronous settings, \citet{aviv2021asynchronous} showed that this mechanism adapts to delays by using
\als
\bw_{t+1}=\proj{\bw_t-\eta\alpha_t\bg_{t-\tau_t}}, \qquad \bx_{t+1}=\bx_t+\frac{\alpha_{t+1}}{\alpha_{1:t+1}}\left(\bw_{t+1}-\bx_t\right)
\eals
with $\alpha_t=t$.
However, their guarantees require delays to be independent of the data.
Instead, by replacing the SGD estimator with STORM in our asynchronous momentum framework, we show that optimal guarantees hold even when the data distribution and delay are dependent.

\subsection{Combining Anytime-SGD and STORM}
We combine Anytime-SGD with STORM by replacing the stochastic gradient in the Anytime update with the STORM estimator:
\als
\bw_{t+1}
&=\proj{\bw_t-\eta\alpha_t\bd_t^{\text{syn}}}, \qquad \bx_{t+1}=\bx_t+\frac{\alpha_{t+1}}{\alpha_{1:t+1}}\left(\bw_{t+1}-\bx_t\right)
\eals
We use increasing weights $\{\alpha_t\}_{t\geq 1}$, with $\alpha_0=0$, and choose the time-varying momentum parameter
\als
\beta_t=1-\frac{\alpha_{t-1}}{\alpha_t}
\eals
Under this choice, the \emph{ideal} STORM estimator admits the following compact representation:
\als
\alpha_t\bd_t^{\text{syn}}=\alpha_{t-1}\bd_{t-1}^{\text{syn}}+\alpha_t\bg_t-\alpha_{t-1}\Tilde{\bg}_{t-1}
\eals
We can unroll this equality and get:
\als
\bd_t^{\text{syn}}=\sum_{k=1}^t\left(\frac{\alpha_k}{\alpha_t}\bg_k-\frac{\alpha_{k-1}}{\alpha_t}\Tilde{\bg}_{k-1}\right)
\eals
Since $\frac{\alpha_k}{\alpha_t}=\prod_{m=k+1}^t(1-\beta_m)$, this form extends \cref{eq:poly-momentum} to time-varying $\beta_t$.
For the asynchronous method, define the ordered momentum increment:
\als
\bs_t:=\alpha_t\bg_t-\alpha_{t-1}\tilde{\bg}_{t-1}
\eals
With this choice, the \textit{ordered} STORM momentum simplifies to:
\al\label{eq:mu2}
\alpha_t\bd_t=\alpha_{t-1}\bd_{t-1}+\bs_{t-\tau_t}
\eal

In general, $\bd^{\text{syn}}_t$ is an unbiased estimator of $\df{\bx_t}$ \citep{cutkosky2019momentum}.
In the asynchronous setting, $\bs_t$ is replaced by a delayed increment $\bs_{t-\tau_t}$ (\cref{eq:mu2}), which introduces bias into the estimator.
As discussed earlier, for the choice of $\alpha_t=t$, the magnitude of the bias of $\alpha_t\bd_t$ scales as $\oo{M}$ (see \cref{sec:non-convex}), which means that the bias of $\bd_t$ is $\oo{\frac{M}{t}}$.

\begin{algorithm}[t]
\caption{Asynchronous $\mu^2$-SGD with Ordered Weights}
\label{alg:mu2}
\begin{algorithmic}
\STATE \textbf{Input:} Initial point $\bx_0\in\K$, step size $\eta>0$, weights $\{\alpha_t>0\}_{t=1}^T$, $\alpha_0=0$.
Set $\bw_1=\bx_1=\bx_0$.
\FOR{$t=1$ to $T$}
\STATE Get $\bg_{t-\tau_t}:=\df{\bx_{t-\tau_t};\bz_{t-\tau_t}}$ and $\Tilde{\bg}_{t-\tau_t-1}:=\df{\bx_{t-\tau_t-1};\bz_{t-\tau_t}}$ from worker.
\STATE Compute: $\bs_{t-\tau_t}=\alpha_{t-\tau_t}\bg_{t-\tau_t}-\alpha_{t-\tau_t-1}\Tilde{\bg}_{t-\tau_t-1}$
\STATE Update: $\alpha_t\bd_t=\alpha_{t-1}\bd_{t-1}+\bs_{t-\tau_t}$
\STATE Update: $\bw_{t+1}=\proj{\bw_t-\eta\alpha_t\bd_t}$
\STATE Update: $\bx_{t+1}=\bx_t+\frac{\alpha_{t+1}}{\alpha_{1:t+1}}\left(\bw_{t+1}-\bx_t\right)$
\STATE Send $\bx_{t+1}$ back to the worker.
\ENDFOR
\STATE \textbf{Output:} $\bx_T$
\end{algorithmic}
\end{algorithm}

\paragraph{Ordered Double Momentum for Data-Independent Delays.}
\begin{theorem}[Convex Smooth Objectives on Data Independent Delays]\label{thm:mu2_indep}
Let $f$ satisfy assumptions~\ref{assume:bound-conv} and \ref{assume:indep}.
Consider $\{\bx_t\}_{t=1}^T$ to be the iterates generated by the update rule according to \cref{alg:mu2} with $\alpha_t=t$, where $\eta$ satisfies $4L\leq\frac{1}{T\eta}\leq\oo{\frac{G^\star}{D}+\left(\frac{\sigma}{D}+\sigmal\right)\sqrt{T}+LM}$.
Then we get an excess loss bound of:
\als
\Exp{\f{\bx_T}}-f^\star\leq\oo{\frac{DG^\star}{T}+\frac{D\sigma+D^2\sigmal}{\sqrt{T}}+\frac{D^2LM}{T}}
\eals
\end{theorem}

\paragraph{Ordered Double Momentum for Data-Dependent Delays.}
\begin{theorem}[Convex Lipschitz-Smooth Objectives on Data-Dependent Delays]\label{thm:mu2_dep}
Let $f$ satisfy assumptions~\ref{assume:bound-conv} and \ref{assume:dep}.
Consider $\{\bx_t\}_{t=1}^T$ to be the iterates generated by the update rule according to \cref{alg:mu2} with $\alpha_t=t$, where $\eta$ satisfies $4L\leq\frac{1}{T\eta}\leq\oo{\left(\frac{\sigma}{D}+\sigmal\right)\sqrt{T}+\left(\frac{G}{D}+L\right)M}$.
Then we get an excess loss bound of:
\als
\Exp{\f{\bx_T}}-f^\star\leq\oo{\frac{D\sigma+D^2\sigmal}{\sqrt{T}}+\frac{D(G+DL)M}{T}}
\eals
\end{theorem}

\vspace{5pt}
\begin{remark}\label{remark:convex-optimality}
\textbf{Optimality under data-dependent delays.}
To the best of our knowledge, \cref{thm:mu2_dep} establishes the first (i) rate for smooth convex objectives under data-dependent delays
and, among existing guarantees in this regime, the first to (ii) exhibit an $M$-independent dominant term and (iii) match the tight rate under data-independent delays.
Indeed, the previous convex guarantee applies to non-smooth objectives and scales as $\oo{DG\sqrt{M/T}}$, incurring a $\oo{\sqrt{M}}$ slowdown in its dominant term \citep{mishchenko2022asynchronous}.
Similarly, the best-known smooth non-convex guarantee under data-dependent delays retains a dependence on $M$ in its dominant term of $\oo{\left({GL\Delta M}/{T}\right)^{2/3}}$.
In contrast, our smooth convex bound confines the dependence on $M$ to the lower-order term $\oo{(D^2L+DG)M/T}$, while its dominant $\oo{1/\sqrt{T}}$ term is independent of $M$.
It therefore matches the best-known data-independent rates (\cref{thm:mu2_indep}; \citealp{aviv2021asynchronous,cohen2021asynchronous}) and recovers their leading-order parallelization benefit despite data-dependent delays.
\end{remark}

\vspace{5pt}
\begin{remark}[Robustness to Learning-Rate Selection]\label{remark:lr-convex}
\Cref{thm:mu2_indep,thm:mu2_dep} show that $\mu^2$-SGD achieves the stated guarantees over a broad range of learning rates.
Specifically, any $\eta\in[\eta_{\min},\eta_{\max}]$ is admissible, and $\eta_{\max}/\eta_{\min}=\ot{\sqrt{T}+M}$.
Thus, even under highly variable asynchrony and data-dependent delays, the method tolerates learning-rate misspecification by a factor of order $\sqrt{T}+M$.
This robustness reduces the need for fine-grained hyperparameter tuning.
\end{remark}

\vspace{5pt}
\begin{remark}[Time complexity]\label{remark:time-complexity}
Using the framework of \citet{tyurin2023optimal}, the iteration-complexity bounds in \cref{thm:momentum_dep,thm:mu2_dep} translate into the first best-known time-complexity guarantees under general data-dependent delays.
Furthermore, when delays are worker-dependent, our method can be combined with a simple subset-selection rule that chooses the fastest workers (see e.g., \citet{maranjyan2025ringmaster}), thereby recovering the optimal time-complexity bounds for worker-dependent delays.
Thus, one can use our framework to combine the advantages of both settings: establishing optimal time complexity for general data-dependent delays and worker-dependent delays.
\end{remark}

\section{Experiments}
\label{sec:exp}

\begin{figure}[t]
\centering
\includegraphics[width=0.8\linewidth]{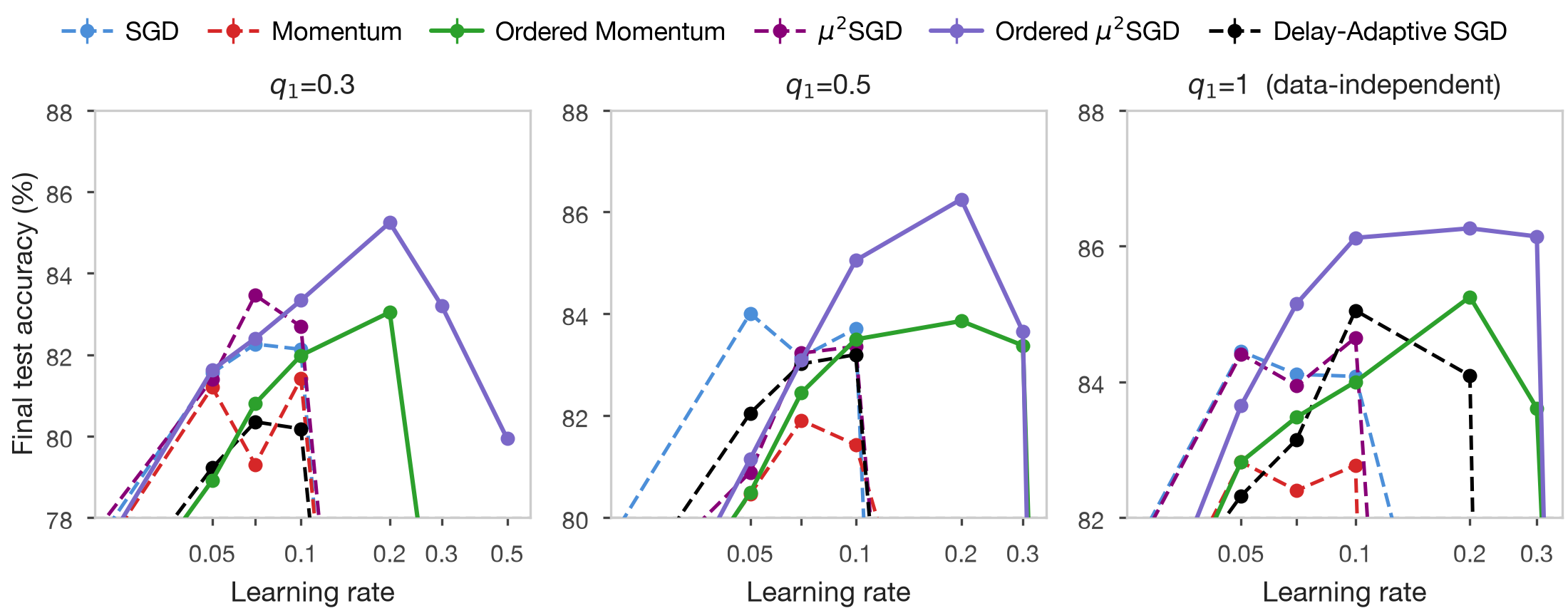}
\caption{\small
\textbf{Learning-rate robustness on Fashion-MNIST with $7$ workers.}
Each point reports the mean final test accuracy.
Both ordered methods remain stable over a wider range of learning rates.}
\label{fig:lr-robustness-w7}
\end{figure}

We evaluate our framework on Fashion-MNIST \citep{xiao2017fashion} under controlled delay distributions (\cref{sec:delay-model}).
We evaluate $M\in\{5,7\}$ workers and vary the proportion of slow, high-quality samples using $q_1\in\{0.3,0.5,1.0\}$.
For $q_1\in\{0.3,0.5\}$, higher-quality samples are assigned longer delays, inducing data-dependent asynchrony.
When $q_1=1.0$, all samples retain their original quality, removing this dependence and providing a data-independent control.
We train a convolutional network and simulate asynchrony according to the delay model in \cref{sec:delay-model}.
We compare Ordered Momentum and Ordered $\mu^2$-SGD against standard asynchronous baselines: Momentum (\cref{eq:naive_asy}; \citealp{polyak1964some}), $\mu^2$-SGD \citep{dahanstochastic}, vanilla SGD (\cref{eq:async-sgd}; \citealp{koloskova2022sharper}), and Delay-Adaptive SGD (\cref{eq:delay-adaptive-sgd}; \citealp{mishchenko2022asynchronous}).
For both $\mu^2$-SGD variants, we adopt the constant-parameter formulation of \citet{dahanstochastic}.
The naive asynchronous variant updates
\al\label{eq:naive-mu2}
\bd_t&=\bg_{t-\tau_t}+(1-\beta)\left(\bd_{t-1}-\tilde{\bg}_{t-\tau_t-1}\right) \non
\bw_{t+1}&=\bw_t-\eta\bd_t, \qquad \bx_{t+1}=\gamma\bw_{t+1}+(1-\gamma)\bx_t
\eal
Ordered $\mu^2$-SGD uses the same constant momentum parameters but replaces the naive update of $\bd_t$ with the ordered STORM update in \cref{eq:ordered-storm-fixed}.
We implement all methods in PyTorch and train for $20$ epochs using a mini-batch size of $1024$.
Each point in the learning-rate sweep reports the final test accuracy, averaged over three random seeds.
All experiments are conducted on an NVIDIA A100 GPU.
Complete experiments, implementation, and hyperparameter details are provided in Appendix~\ref{app:exp}.
Results for both worker counts are reported in \cref{fig:lr-robustness-w7,fig:lr-robustness-w5-nostd}.

\paragraph{Results.} Across both worker counts and all three values of $q_1$, ordering consistently improves momentum-based methods.
Ordered Momentum and Ordered $\mu^2$-SGD remain effective across a wider range of learning rates than their unordered counterparts and asynchronous SGD baselines.
This behavior is consistent with the learning-rate robustness guarantees established for the non-convex and convex settings in Remark~\ref{remark:practical} and Remark~\ref{remark:lr-convex}.
Therefore, the results demonstrate that ordering can enhance the robustness of momentum and enables more stable optimization under both data-independent and data-dependent asynchrony.

\section{Conclusions and Future Work}
This work introduces a unified momentum-based framework for asynchronous optimization that achieves the best-known convergence rates under both data-independent and data-dependent delays.
In particular, we established the first convergence guarantee for smooth convex optimization with data-dependent delays and showed that it matches the tight rate known for the data-independent setting.
Our framework uses a fixed, delay-independent learning rate, simplifying hyperparameter tuning, even under highly variable asynchrony.
Future work includes extending this framework to adaptive optimizers such as Adam and AdamW, as well as related variance-reduced methods, to enable more robust large-scale asynchronous training.


\bibliographystyle{plainnat}
\bibliography{ordered_momentum/bib_}

\clearpage
\appendix

\section{Technical Preliminaries}
\label{app:tech}

\paragraph{Implications Among the Assumptions.}

\begin{proof}
First, the bounded gradient assumption implies bounded stochastic noise because
\als
\Exp{\normsq{\df{\bx;\bz}-\df{\bx}}}\leq\Exp{\normsq{\df{\bx;\bz}}}\leq G^2
\eals
where we used $\Exp{\normsq{X-\ExpB{X}}}\leq\Exp{\normsq{X}}$.
Thus, one may take $\sigma\leq G$.

Next, sample smoothness implies the centered increment bound
\als
\Exp{\normsq{\left(\df{\bx;\bz}-\df{\bx}\right)-\left(\df{\by;\bz}-\df{\by}\right)}}\leq\Exp{\normsq{\df{\bx;\bz}-\df{\by;\bz}}}\leq L^2\normsq{\bx-\by}
\eals
by the same variance inequality and sample smoothness.
Hence, one may take $\sigmal\leq L$.

Finally, Jensen's inequality and sample smoothness give
\als
\normsq{\df{\bx}-\df{\by}}\leq\Exp{\normsq{\df{\bx;\bz}-\df{\by;\bz}}}\leq L^2\normsq{\bx-\by}
\eals
so the expected objective is also $L$-smooth.
\end{proof}

Smoothness can help us bound the gradient, but we get a different bound in each case.

\begin{lemma}\label{lem:smooth-functions}
Let $f:\real^d\to\reals$ be a function that satisfies assumption~\ref{assume:non-conv}, then $\forall\bx\in\real^d$ we have:
\als
\normsq{\df{\bx}}\leq2L\left(\f{\bx}-f^\star\right)
\eals
\end{lemma}

\begin{proof}[Proof of Lemma~\ref{lem:smooth-functions}]
Using $L$-smoothness, we get:
\als
\f{\by}-\f{\bx}\leq\dotprod{\df{\bx}}{\by-\bx}+\frac{L}{2}\normsq{\by-\bx}=-\frac{1}{2L}\normsq{\df{\bx}}
\eals
Where we get the equality by setting $\by=\bx-\frac{1}{L}\df{\bx}$.
Using $f(\by)\geq f^\star$ and a simple rearrangement yields the claim.
\end{proof}
The use of this lemma is to bound $\normsq{\df{\bx_1}}\leq2L\Delta$.

\begin{lemma}\label{lem:smooth-conv}
Let $f:\K\to\reals$ be a function that satisfies assumption~\ref{assume:bound-conv}, then $\forall\bx\in\K$ we have:
\als
\normsq{\df{\bx}-\df{\bxs}}\leq2L(\f{\bx}-f^\star)
\eals
\end{lemma}

\begin{proof}[Proof of Lemma~\ref{lem:smooth-conv}]
Let us define a new function:
\als
h(\bx)=\f{\bx}-f^\star-\dotprod{\df{\bxs}}{\bx-\bxs}
\eals
Since $f$ is convex and $L$-smooth, we know that:
\als
0\leq h(\bx)\leq\frac{L}{2}\normsq{\bx-\bxs}
\eals
The gradient of this function is:
\als
\nabla h(\bx)=\df{\bx}-\df{\bxs}
\eals
We can see that $h(\bxs)=0, \nabla h(\bxs)=0$, and that $\bxs$ is the global minimum.
The function $h$ is also convex and $L$-smooth, since the gradient is the same as $f$ up to a constant translation.
We will add to the domain of $h$ to include all $\reals^d$, while still being convex and $L$-smooth.
Since $h$ is convex then:
\als
h(\by)\geq h(\bxs)+\dotprod{\nabla h(\bxs)}{\by-\bxs}=0,\quad\forall \by\in\reals^d
\eals
It is true even for points outside of the original domain, meaning that $\bxs$ remains the global minimum even after this.
For a smooth function, $\forall\bx,\by\in\reals^d$:
\als
h(\by)\leq h(\bx)+\dotprod{\nabla h(\bx)}{\by-\bx}+\frac{L}{2}\normsq{\by-\bx}
\eals
By picking $\by=\bx-\frac{1}{L}\nabla h(\bx)$, we get:
\als
h(\bx)-h(\by)\geq\frac{1}{2L}\normsq{\nabla h(\bx)}
\eals
Rearranging, we get:
\als
\normsq{\nabla h(\bx)}\leq2L(h(\bx)-h(\by))\leq2L\cdot h(\bx)
\eals
By using $\bx\in\K$ we get:
\als
\normsq{\df{\bx}-\df{\bxs}}\leq2L\left(\f{\bx}-f^\star-\dotprod{\df{\bxs}}{\bx-\bxs}\right)
\eals
Since $\bxs$ is the minimum point of $f$ then:
\als
\dotprod{\df{\bxs}}{\bx-\bxs}\geq0,\quad\forall\bx\in\K
\eals
Thus we get that:
\als
\normsq{\df{\bx}-\df{\bxs}}\leq2L(\f{\bx}-f^\star)
\eals
\end{proof}

\section{Non-Convex Analysis}\label{app:momentum}
\paragraph{Ideal Momentum.}
In the spirit of \citet{mishchenko2022asynchronous,shi2024ordered}, we define the \emph{ideal} momentum $\hbm_t$ (\cref{eq:moment_ideal}), in which all updates are applied exactly at their designated iterations.

We will now bound the momentum:
\begin{lemma}\label{lem:reg_m}
Let $f:\real^d\to\reals$ be a function that satisfies assumption~\ref{assume:non-conv}, and define the stochastic error as $\beps_t:=\bbm_t-\df{\bx_t}$.
Consider the update rule $\bx_{t+1}:=\bx_t-\eta\bbm_t$ with a learning rate $\eta\leq\frac{1}{2L}$, then: 
\als
\frac{1}{T}\sum_{t=1}^T\Exp{\normsq{\bbm_t}}\leq\frac{4\Delta}{\eta T}+\frac{2}{T}\sum_{t=1}^T\Exp{\normsq{\beps_t}}
\eals
\end{lemma}

\begin{proof}[Proof of Lemma~\ref{lem:reg_m}]
By the smoothness property of $f$, it follows since,
\als
\f{\bx_{t+1}}-\f{\bx_t}\leq&\dotprod{\df{\bx_t}}{\bx_{t+1}-\bx_t}+\frac{L}{2}\normsq{\bx_{t+1}-\bx_t}=-\eta\dotprod{\df{\bx_t}}{\bbm_t}+\frac{L\eta^2}{2}\normsq{\bbm_t} \\
=&-\eta\normsq{\bbm_t}+\eta\dotprod{\beps_t}{\bbm_t}+\frac{L\eta^2}{2}\normsq{\bbm_t} \\
\leq&-\eta\normsq{\bbm_t}+\frac{\eta}{2}\normsq{\beps_t}+\frac{\eta}{2}\normsq{\bbm_t}+\frac{L\eta^2}{2}\normsq{\bbm_t}\leq-\frac{\eta}{4}\normsq{\bbm_t}+\frac{\eta}{2}\normsq{\beps_t}
\eals
Here, the first equality employs the definition of the update rule for $\bx_{t+1}$ with learning rate $\eta$, and the second follows the definition of the stochastic error $\beps_t:=\bbm_t-\df{\bx_t}$.
The next inequality uses Young's inequality $\dotprod{\ba}{\bb}\leq\frac{1}{2}\normsq{\ba}+\frac{1}{2}\normsq{\bb}, \forall\ba,\bb\in\real^d$ and the last inequality holds by setting $\eta\leq\frac{1}{2L}$.

Rearranging the above, we obtain:
\als
\normsq{\bbm_t}\leq4\frac{\f{\bx_t}-\f{\bx_{t+1}}}{\eta}+2\normsq{\beps_t}
\eals
Averaging over all $T$ iterations and taking the expectation gives us:
\als
\frac{1}{T}\sum_{t=1}^T\Exp{\normsq{\bbm_t}}\leq\frac{4}{\eta T}\sum_{t=1}^T\ExpB{\f{\bx_t}-\f{\bx_{t+1}}}+\frac{2}{T}\sum_{t=1}^T\Exp{\normsq{\beps_t}}\leq\frac{4\Delta}{\eta T}+\frac{2}{T}\sum_{t=1}^T\Exp{\normsq{\beps_t}}
\eals
\end{proof}

A similar bound will be used for the gradients:
\begin{lemma}\label{lem:reg_g}
Let $f:\real^d\to\reals$ be a function that satisfies assumption~\ref{assume:non-conv}, and define the stochastic error as $\beps_t:=\bbm_t-\df{\bx_t}$.
Consider the update rule $\bx_{t+1}:=\bx_t-\eta\bbm_t$ with a learning rate $\eta\leq\frac{1}{L}$, then: 
\als
\frac{1}{T}\sum_{t=1}^T\Exp{\normsq{\df{\bx_t}}}\leq\frac{2\Delta}{\eta T}+\frac{1}{T}\sum_{t=1}^T\Exp{\normsq{\beps_t}}
\eals
\end{lemma}

\begin{proof}[Proof of Lemma~\ref{lem:reg_g}]
By the smoothness property of $f$, it follows since,
\als
\f{\bx_{t+1}}-\f{\bx_t}\leq&\dotprod{\df{\bx_t}}{\bx_{t+1}-\bx_t}+\frac{L}{2}\normsq{\bx_{t+1}-\bx_t} \\
=&-\eta\dotprod{\df{\bx_t}}{\df{\bx_t}+\beps_t}+\frac{L\eta^2}{2}\normsq{\df{\bx_t}+\beps_t} \\
=&\frac{\eta}{2}\left(-(2-L\eta)\normsq{\df{\bx_t}}-2(1-L\eta)\dotprod{\df{\bx_t}}{\beps_t}+L\eta\normsq{\beps_t}\right) \\
\leq&\frac{\eta}{2}\left(((1-L\eta)-(2-L\eta))\normsq{\df{\bx_t}}+((1-L\eta)+L\eta)\normsq{\beps_t}\right) \\
=&-\frac{\eta}{2}\normsq{\df{\bx_t}}+\frac{\eta}{2}\normsq{\beps_t}
\eals
Here, the first equality employs the definition of the update rule for $\bx_{t+1}$ with learning rate $\eta$, and the second follows the definition of the stochastic error $\beps_t:=\bbm_t-\df{\bx_t}$.
The second inequality uses Young's inequality $\dotprod{\ba}{\bb}\leq\frac{1}{2}\normsq{\ba}+\frac{1}{2}\normsq{\bb}, \forall\ba,\bb\in\real^d$ and remembering that $1-L\eta\geq0$.

Rearranging the above, we obtain:
\als
\normsq{\df{\bx_t}}\leq2\frac{\f{\bx_t}-\f{\bx_{t+1}}}{\eta}+\normsq{\beps_t}
\eals
Averaging over all $T$ iterations and taking the expectation gives us:
\als
\frac{1}{T}\sum_{t=1}^T\Exp{\normsq{\df{\bx_t}}}\leq\frac{2}{\eta T}\sum_{t=1}^T\ExpB{\f{\bx_t}-\f{\bx_{t+1}}}+\frac{1}{T}\sum_{t=1}^T\Exp{\normsq{\beps_t}}\leq\frac{2\Delta}{\eta T}+\frac{1}{T}\sum_{t=1}^T\Exp{\normsq{\beps_t}}
\eals
\end{proof}

\begin{lemma}\label{lem:mean_error}
We will define the expected weighted gradient $\bbg_{t,k}:=\beta(1-\beta)^{t-k}\df{\bx_k}$, and the expected error $\bbeps_t:=\sum_{k=1}^t\bbg_{t,k}-\df{\bx_t}$.
Then:
\als
\frac{1}{T}\sum_{t=1}^T\normsq{\bbeps_t}\leq\frac{2L\Delta}{\beta T}+\frac{L^2\eta^2}{\beta^2T}\sum_{t=1}^T\normsq{\bbm_t}
\eals
\end{lemma}

\begin{proof}[Proof of Lemma~\ref{lem:mean_error}]
To bound the expected error we will find its update step:
\als
\bbeps_t=&\sum_{k=1}^t\bbg_{t,k}-\df{\bx_t}=(1-\beta)\sum_{k=1}^{t-1}\bbg_{t-1,k}-(1-\beta)\df{\bx_t} \\
=&(1-\beta)\left(\bbeps_{t-1}+\df{\bx_{t-1}}-\df{\bx_t}\right)
\eals
Where we used $\bbg_{t,t}=\beta\df{\bx_t}$ and $\bbg_{t,k}=(1-\beta)\bbg_{t-1,k}$.
Using this update rule, we get:
\als
\normsq{\bbeps_{t+1}}=&(1-\beta)^2\normsq{\bbeps_t+\df{\bx_t}-\df{\bx_{t+1}}} \\
\leq&(1-\beta)\normsq{\bbeps_t}+\frac{(1-\beta)^2}{\beta}\normsq{\df{\bx_t}-\df{\bx_{t+1}}} \\
\leq&(1-\beta)\normsq{\bbeps_t}+\frac{(1-\beta)^2L^2}{\beta}\normsq{\bx_t-\bx_{t+1}}\leq(1-\beta)\normsq{\bbeps_t}+\frac{L^2\eta^2}{\beta}\normsq{\bbm_t}
\eals
Where we used the inequality $\normsq{\bx+\by}\leq\frac{1}{p}\normsq{\bx}+\frac{1}{1-p}\normsq{\by}, \forall\bx,\by\in\real^d,p\in(0,1)$, with $p=1-\beta$, the smoothness of $f$, and the update rule $\bx_{t+1}=\bx_t-\eta\bbm_t$.
Rearranging the terms and summing over from 1 to $T-1$ yields:
\als
\normsq{\bbeps_T}-\normsq{\bbeps_1}+\beta\sum_{t=1}^{T-1}\normsq{\bbeps_t}\leq\frac{L^2\eta^2}{\beta}\sum_{t=1}^{T-1}\normsq{\bbm_t}\leq\frac{L^2\eta^2}{\beta}\sum_{t=1}^T\normsq{\bbm_t}
\eals
Where we added another term to the sum at the end.
Since $\bbeps_1=-(1-\beta)\df{\bx_1}$, we get:
\als
\beta\sum_{t=1}^T\normsq{\bbeps_t}\leq&(1-\beta)^2\normsq{\df{\bx_1}}-(1-\beta)\normsq{\bbeps_T}+\frac{L^2\eta^2}{\beta}\sum_{t=1}^T\normsq{\bbm_t} \\
\leq&2L\Delta+\frac{L^2\eta^2}{\beta}\sum_{t=1}^T\normsq{\bbm_t}
\eals
Dividing by $\beta T$:
\als
\frac{1}{T}\sum_{t=1}^T\normsq{\bbeps_t}\leq\frac{2L\Delta}{\beta T}+\frac{L^2\eta^2}{\beta^2T}\sum_{t=1}^T\normsq{\bbm_t}
\eals
\end{proof}

\paragraph{Data Independent Case.}
We will bound the error in the data independent case:
\begin{lemma}\label{lem:momentum_indep}
Let $f:\real^d\to\reals$ be a function that satisfies assumptions~\ref{assume:non-conv} and \ref{assume:indep}.
Then, for any iteration $t\geq1$ with momentum parameter $\beta\leq\frac{1}{12(M-1)}$, the cumulative stochastic error of the delayed momentum in~\Cref{alg:momentum}, $\beps_t:=\bbm_t-\df{\bx_t}$, evaluated over $T$ iterations, satisfies:
\als
\frac{1}{T}\sum_{t=1}^T\Exp{\normsq{\beps_t}}\leq\frac{6L\Delta}{\beta T}+3\beta\sigma^2+\frac{3L^2\eta^2}{\beta^2T}\sum_{t=1}^T\Exp{\normsq{\bbm_t}}+\frac{1}{4T}\sum_{t=1}^T\Exp{\normsq{\df{\bx_t}}}
\eals
\end{lemma}

\begin{proof}[Proof of Lemma~\ref{lem:momentum_indep}]
Since the data is independent of the delay, we can easily separate the bias and variance of $\beps_t$, given $\T(t)$:
\als
\beps_t=\sum_{k\in[t]\setminus\T(t)}\bg_{t,k}-\df{\bx_t}=\bbeps_t+\sum_{k\in\T(t)}\bbg_{t,k}+\left(\sum_{k\in[t]\setminus\T(t)}\left(\bg_{t,k}-\bbg_{t,k}\right)\right)
\eals
We can use $\normsq{\ba+\bb+\bc}\leq3\normsq{\ba}+3\normsq{\bb}+3\normsq{\bc}$ to obtain:
\als
\Exp{\normsq{\beps_t}}=3\Exp{\normsq{\bbeps_t}}+3\Exp{\normsq{\sum_{k\in\T(t)}\bbg_{t,k}}}+3\sum_{k\in[t]\setminus\T(t)}\Exp{\normsq{\bg_{t,k}-\bbg_{t,k}}}
\eals
Where the last term can be split because of the data-independency assumption.
The last term can be bounded using the bounded variance assumption:
\als
\sum_{k\in[t]\setminus\T(t)}\Exp{\normsq{\bg_{t,k}-\bbg_{t,k}}}\leq\beta^2\sigma^2\sum_{k\in[t]\setminus\T(t)}(1-\beta)^{2(t-k)}\leq\beta^2\sigma^2\sum_{k=1}^t(1-\beta)^{2(t-k)}\leq\beta\sigma^2
\eals
Where at the end, we bounded the geometric sum.
We will bound the second term:
\als
\normsq{\sum_{k\in\T(t)}\bbg_{t,k}}\leq&(M-1)\sum_{k\in\T(t)}\normsq{\bbg_{t,k}}\leq(M-1)\beta^2\sum_{k=1}^t(1-\beta)^{2(t-k)}\normsq{\df{\bx_k}} \\
\leq&\frac{\beta}{8}\sum_{k=1}^t(1-\beta)^{2(t-k)}\normsq{\df{\bx_k}}
\eals
Where the first inequality uses Jensen's inequality and the fact that $|\T(t)|\leq M-1$.
In the next inequality we input the definitions of $\bbg_{t,k}$ and add the missing members in the sum.
At last, we use the bound $(M-1)\beta\leq\frac{1}{12}$.
In total, we get:
\als
\Exp{\normsq{\beps_t}}\leq3\Exp{\normsq{\bbeps_t}}+\frac{\beta}{4}\sum_{k=1}^t(1-\beta)^{2(t-k)}\Exp{\normsq{\df{\bx_k}}}+3\beta\sigma^2
\eals
Averaging over $t$, we get:
\als
\frac{1}{T}\sum_{t=1}^T\Exp{\normsq{\beps_t}}\leq\frac{3}{T}\sum_{t=1}^T\Exp{\normsq{\bbeps_t}}+\frac{\beta}{4T}\sum_{t=1}^T\sum_{k=1}^t(1-\beta)^{2(t-k)}\Exp{\normsq{\df{\bx_k}}}+3\beta\sigma^2
\eals
The first terms can be bounded using Lemma~\ref{lem:mean_error}.
For the second term, swapping the order of the sum and solving the geometric sum, gives us $\frac{1}{4T}\sum_{t=1}^T\Exp{\normsq{\df{\bx_t}}}$.
In total, we get:
\als
\frac{1}{T}\sum_{t=1}^T\Exp{\normsq{\beps_t}}\leq\frac{6L\Delta}{\beta T}+3\beta\sigma^2+\frac{3L^2\eta^2}{\beta^2T}\sum_{t=1}^T\Exp{\normsq{\bbm_t}}+\frac{1}{4T}\sum_{t=1}^T\Exp{\normsq{\df{\bx_t}}}
\eals
\end{proof}

\begin{proof}[Proof of \cref{thm:momentum_indep}]
Starting from Lemma~\ref{lem:momentum_indep}, we input Lemmas~\ref{lem:reg_m} and~\ref{lem:reg_g}:
\als
\frac{1}{T}\sum_{t=1}^T\Exp{\normsq{\beps_t}}\leq&\frac{6L\Delta}{\beta T}+3\beta\sigma^2+\frac{3L^2\eta^2}{\beta^2T}\sum_{t=1}^T\Exp{\normsq{\bbm_t}}+\frac{1}{4T}\sum_{t=1}^T\Exp{\normsq{\df{\bx_t}}} \\
\leq&\frac{6L\Delta}{\beta T}+3\beta\sigma^2+\frac{12L^2\Delta\eta}{\beta^2T}+\frac{6L^2\eta^2}{\beta^2T}\sum_{t=1}^T\Exp{\normsq{\beps_t}}+\frac{\Delta}{2\eta T}+\frac{1}{4T}\sum_{t=1}^T\Exp{\normsq{\beps_t}}
\eals
Using $\eta\leq\frac{\beta}{2\sqrt{6}L}$, we get:
\als
\frac{1}{T}\sum_{t=1}^T\Exp{\normsq{\beps_t}}\leq\frac{6L\Delta}{\beta T}+3\beta\sigma^2+\frac{12L^2\Delta\eta}{\beta^2T}+\frac{\Delta}{2\eta T}+\frac{1}{2T}\sum_{t=1}^T\Exp{\normsq{\beps_t}}
\eals
Self bounding gives us:
\als
\frac{1}{T}\sum_{t=1}^T\Exp{\normsq{\beps_t}}\leq\frac{12L\Delta}{\beta T}+6\beta\sigma^2+\frac{24L^2\Delta\eta}{\beta^2T}+\frac{\Delta}{\eta T}
\eals
Inputting this result in Lemma~\ref{lem:reg_g}, gives us:
\als
\frac{1}{T}\sum_{t=1}^T\Exp{\normsq{\df{\bx_t}}}\leq\frac{2\Delta}{\eta T}+\frac{1}{T}\sum_{t=1}^T\Exp{\normsq{\beps_t}}\leq\frac{12L\Delta}{\beta T}+6\beta\sigma^2+\frac{24L^2\Delta\eta}{\beta^2T}+\frac{3\Delta}{\eta T}
\eals
Inputting $\eta=\frac{\beta}{2\sqrt{6}L}$, we get:
\als
\frac{1}{T}\sum_{t=1}^T\Exp{\normsq{\df{\bx_t}}}\leq\frac{4(3+2\sqrt{6})L\Delta}{\beta T}+6\beta\sigma^2\leq\frac{32L\Delta}{\beta T}+6\beta\sigma^2
\eals
Where we used $\sqrt{6}\leq5/2$.
Picking $\beta=\mini{\frac{1}{12(M-1)},\frac{4\sqrt{L\Delta}}{\sigma\sqrt{3T}}}$, we get:
\als
\frac{1}{T}\sum_{t=1}^T\Exp{\normsq{\df{\bx_t}}}\leq\frac{384L\Delta(M-1)}{T}+\frac{16\sigma\sqrt{3L\Delta}}{\sqrt{T}}
\eals
\end{proof}

\paragraph{Data Dependent Case.}
We will bound the error in the data dependent case:
\begin{lemma}\label{lem:momentum_dep}
Let $f:\real^d\to\reals$ be a function that satisfies assumptions~\ref{assume:non-conv} and \ref{assume:dep}.
Then, for any iteration $t\geq1$ with momentum parameter $\beta\in(0,1)$, the cumulative stochastic error of the momentum in~\cref{alg:momentum}, evaluated over $T$ iterations, satisfies:
\als
\frac{1}{T}\sum_{t=1}^T\Exp{\normsq{\beps_t}}\leq\frac{6L\Delta}{\beta T}+3\beta\sigma^2+3\beta^2(M-1)^2G^2+\frac{3L^2\eta^2}{\beta^2T}\sum_{t=1}^T\Exp{\normsq{\bbm_t}}
\eals
\end{lemma}

\begin{proof}[Proof of Lemma~\ref{lem:momentum_dep}]
We will define the delay $\bb_t:=\bbm_t-\hbm_t=\beps_t-\hbeps_t$, and note~\cref{eq:moment_delay}:
\als
\bb_t=-\sum_{k\in\T(t)}\bg_{t,k}
\eals
Using~\cref{eq:moment_ideal}, we know that the ideal stochastic error is:
\als
\hbeps_t:=\hbm_t-\df{\bx_t}=\sum_{k=1}^t\bg_{t,k}-\df{\bx_t}
\eals
We will define the expected stochastic error as:
\als
\bbeps_t=\sum_{k=1}^t\bbg_{t,k}-\df{\bx_t}
\eals
Next, we use the decomposition $\beps_t=\bbeps_t+(\hbeps_t-\bbeps_t)+\bb_t$ and the inequality $\normsq{\ba+\bb+\bc}\leq3\normsq{\ba}+3\normsq{\bb}+3\normsq{\bc}$, we obtain $\Exp{\normsq{\beps_t}}\leq3\Exp{\normsq{\bbeps_t}}+3\Exp{\normsq{\hbeps_t-\bbeps_t}}+3\Exp{\normsq{\bb_t}}$.
\als
\norm{\bb_t}=&\norm{-\sum_{k\in\T(t)}\bg_{t,k}}\leq\sum_{k\in\T(t)}\norm{\bg_{t,k}}=\sum_{k\in\T(t)}\beta(1-\beta)^{t-k}\norm{\df{\bx_k;\bz_k}} \\
\leq&\beta G\sum_{k\in\T(t)}(1-\beta)^{t-k}\leq\beta(M-1)G
\eals
Where the first inequality is the triangle inequality, the second inequality is the Lipchitz assumption, and in the last inequality we bound $(1-\beta)\leq1$ and $|\T(t)|\leq M-1$.

Since $\hbeps_t-\bbeps_t$ has all the terms, it is not affected by the delay, so we don't care that the delay depends on the data.
Thus, it is a martingale sequence:
\als
&\Exp{\normsq{\hbeps_t-\bbeps_t}}=\Exp{\normsq{\sum_{k=1}^t\left(\bg_{t,k}-\bbg_{t,k}\right)}}=\sum_{k=1}^t\Exp{\normsq{\bg_{t,k}-\bbg_{t,k}}} \\
=&\sum_{k=1}^t\beta^2(1-\beta)^{2(t-k)}\Exp{\normsq{\bg_k-\df{\bx_k}}}\leq\beta^2\sigma^2\sum_{k=1}^t(1-\beta)^{2(t-k)}\leq\beta\sigma^2
\eals
Where in the first inequality we used the bounded variance, and finally bounded the geometric sum.
In total, we get:
\als
\Exp{\normsq{\beps_t}}\leq3\Exp{\normsq{\bbeps_t}}+3\beta\sigma^2+3\beta^2(M-1)^2G^2
\eals
Averaging over $t$ and using Lemma~\ref{lem:mean_error}, we get:
\als
\frac{1}{T}\sum_{t=1}^T\Exp{\normsq{\beps_t}}\leq\frac{6L\Delta}{\beta T}+3\beta\sigma^2+3\beta^2(M-1)^2G^2+\frac{3L^2\eta^2}{\beta^2T}\sum_{t=1}^T\Exp{\normsq{\bbm_t}}
\eals
\end{proof}

\begin{proof}[Proof of~\Cref{thm:momentum_dep}]
Starting from Lemma~\ref{lem:momentum_dep}, we input Lemma~\ref{lem:reg_m}:
\als
\frac{1}{T}\sum_{t=1}^T\Exp{\normsq{\beps_t}}\leq&\frac{6L\Delta}{\beta T}+3\beta\sigma^2+3\beta^2(M-1)^2G^2+\frac{3L^2\eta^2}{\beta^2T}\sum_{t=1}^T\Exp{\normsq{\bbm_t}} \\
\leq&\frac{6L\Delta}{\beta T}+3\beta\sigma^2+3\beta^2(M-1)^2G^2+\frac{12L^2\Delta\eta}{\beta^2T}+\frac{6L^2\eta^2}{\beta^2T}\sum_{t=1}^T\Exp{\normsq{\beps_t}}
\eals
Using $\eta\leq\frac{\beta}{2\sqrt{3}L}$ will yield the following:
\als
\frac{1}{T}\sum_{t=1}^T\Exp{\normsq{\beps_t}}\leq\frac{6L\Delta}{\beta T}+3\beta\sigma^2+3\beta^2(M-1)^2G^2+\frac{12L^2\Delta\eta}{\beta^2T}+\frac{1}{2T}\sum_{t=1}^T\Exp{\normsq{\beps_t}}
\eals
Moving sides and multiplying by 2:
\als
\frac{1}{T}\sum_{t=1}^T\Exp{\normsq{\beps_t}}\leq\frac{12L\Delta}{\beta T}+6\beta\sigma^2+6\beta^2(M-1)^2G^2+\frac{24L^2\Delta\eta}{\beta^2T}
\eals
Inputting this result in Lemma~\ref{lem:reg_g}, gives us:
\als
\frac{1}{T}\sum_{t=1}^T\Exp{\normsq{\df{\bx_t}}}\leq&\frac{2\Delta}{\eta T}+\frac{1}{T}\sum_{t=1}^T\Exp{\normsq{\beps_t}} \\
\leq&\frac{12L\Delta}{\beta T}+6\beta\sigma^2+6\beta^2(M-1)^2G^2+\frac{24L^2\Delta\eta}{\beta^2T}+\frac{2\Delta}{\eta T}
\eals
Inputting $\eta=\frac{\beta}{2\sqrt{3}L}$, we get:
\als
\frac{1}{T}\sum_{t=1}^T\Exp{\normsq{\df{\bx_t}}}\leq\frac{4(3+2\sqrt{3})L\Delta}{\beta T}+6\beta\sigma^2+6\beta^2(M-1)^2G^2
\eals
Bounding $\sqrt{3}\leq7/4$, we get:
\als
\frac{1}{T}\sum_{t=1}^T\Exp{\normsq{\df{\bx_t}}}\leq\frac{26L\Delta}{\beta T}+6\beta\sigma^2+6\beta^2(M-1)^2G^2
\eals
Choosing $\beta=\mini{\frac{1}{2},\frac{\sqrt{13L\Delta}}{\sigma\sqrt{3T}},\left(\frac{13L\Delta}{6G^2(M-1)^2T}\right)^{1/3}}$, we get:
\als
\frac{1}{T}\sum_{t=1}^T\Exp{\normsq{\df{\bx_t}}}\leq\frac{52L\Delta}{T}+\frac{4\sigma\sqrt{39L\Delta}}{\sqrt{T}}+3\left(\frac{13\sqrt{6}GL\Delta(M-1)}{T}\right)^{2/3}
\eals
The condition $\beta\leq0.5$ added to its value is to guarantee that $\beta<1$, which was assumed throughout the proof, since none of the other conditions on $\beta$ guarantee it.
Using $4\sqrt{39}\leq25$ and $13\sqrt{6}\leq32$, we get:
\als
\frac{1}{T}\sum_{t=1}^T\Exp{\normsq{\df{\bx_t}}}\leq\frac{52L\Delta}{T}+\frac{25\sigma\sqrt{L\Delta}}{\sqrt{T}}+12\left(\frac{4GL\Delta(M-1)}{T}\right)^{2/3}
\eals
\end{proof}

\section{Convex Analysis}\label{sec:convex-app}

\begin{theorem}[\citet{cutkosky2019anytime}]\label{thm:anytime}
Let $f:\K\to\reals$ be a convex function.
Also let $\{\alpha_t>0\}_t$ and $\{\bw_t\in\K\}_t$.
Let $\{\bx_t\}$ be the $\{\alpha_\tau\}_{\tau=1}^t$ weighted average of $\{\bw_\tau\}_{\tau=1}^t$, meaning: $\alpha_{1:t}\bx_t=\sum_{\tau=1}^t\alpha_\tau\bw_\tau$.
Then the following holds for all $t\geq1, x\in\K$:
\als
\alpha_{1:t}(\f{\bx_t}-\f{\bx})\leq\sum_{\tau=1}^t\alpha_\tau\dotprod{\df{\bx_\tau}}{\bw_\tau-\bx}
\eals
\end{theorem}
Concretely, the theorem implies that the excess loss of the weighted average $\bx_t$ can be related to the weighted regret $\sum_{\tau=1}^t\alpha_\tau\dotprod{\df{\bx_\tau}}{\bw_\tau-\bx}$.

\begin{proof}[Proof of~\cref{thm:anytime}]
We will prove by induction.

\textbf{Induction Basis:}
For $t=1$, we get that:
\als
\alpha_1(\f{\bx_1}-\f{\bx})\leq\alpha_1\dotprod{\df{\bx_1}}{\bx_1-\bx}=\alpha_1\dotprod{\df{\bx_1}}{\bw_1-\bx}
\eals
Where the inequality is the convexity of $f$, and the equality is because $\bw_1=\bx_1$.

\textbf{Induction Assumption:}
We will assume that~\cref{thm:anytime} hold true for some $t\geq1$, meaning:
\als
\alpha_{1:t}(\f{\bx_t}-\f{\bx})\leq\sum_{\tau=1}^t\alpha_\tau\dotprod{\df{\bx_\tau}}{\bw_\tau-\bx}
\eals

\textbf{Induction Step:}
We will prove for $t+1$.
\als
\alpha_{1:t+1}(\f{\bx_{t+1}}-\f{\bx})=&\alpha_{1:t}(\f{\bx_t}-\f{\bx}) \\
+&\alpha_{1:t}(\f{\bx_{t+1}}-\f{\bx_t})+\alpha_{t+1}(\f{\bx_{t+1}}-\f{\bx})
\eals
We got this by adding and subtracting $\alpha_{1:t}\f{\bx_t}$.
The first term can be bounded by the induction assumption, and the second and third step will each be bounded using the convexity of $f$.
\als
&\alpha_{1:t+1}(\f{\bx_{t+1}}-\f{\bx})=\sum_{\tau=1}^t\alpha_\tau\dotprod{\df{\bx_\tau}}{\bw_\tau-\bx} \\
+&\alpha_{1:t}\dotprod{\df{\bx_{t+1}}}{\bx_{t+1}-\bx_t}+\alpha_{t+1}\dotprod{\df{\bx_{t+1}}}{\bx_{t+1}-\bx}
\eals
Combining the last two terms, we get:
\als
&\alpha_{1:t}\dotprod{\df{\bx_{t+1}}}{\bx_{t+1}-\bx_t}+\alpha_{t+1}\dotprod{\df{\bx_{t+1}}}{\bx_{t+1}-\bx} \\
=&\dotprod{\df{\bx_{t+1}}}{\alpha_{1:t+1}\bx_{t+1}-\alpha_{1:t}\bx_t-\alpha_{t+1}\bx}=\alpha_{t+1}\dotprod{\df{\bx_{t+1}}}{\bw_{t+1}-\bx}
\eals
Where in the last term we used $\alpha_{1:t+1}\bx_{t+1}-\alpha_{1:t}\bx_t=\alpha_{t+1}\bw_{t+1}$.
In total, we got:
\als
\alpha_{1:t+1}(\f{\bx_{t+1}}-\f{\bx})\leq\sum_{\tau=1}^{t+1}\alpha_\tau\dotprod{\df{\bx_\tau}}{\bw_\tau-\bx}
\eals
We proved the induction step.
\end{proof}

We will denote the weighted momentum $\bq_t:=\alpha_t\bd_t$.

\begin{lemma}\label{lem:OGD+_Guarantees}
Let $\eta>0$, and $\K\subset\reals^d$ be a convex domain of bounded diameter $D$, also let $\{\bq_t\in\reals^d\}_{t=1}^T$ be a sequence of arbitrary vectors.
Then for any starting point $\bw_1\in\K$, and an update rule $\bw_{t+1}=\proj{\bw_t-\eta\bq_t},~\forall t\geq 1$, the following holds $\forall\bx\in\K$:
\als
\sum_{\tau=1}^t\dotprod{\bq_\tau}{\bw_{\tau+1}-\bx}\leq\frac{D^2}{2\eta}-\frac{1}{2\eta}\sum_{\tau=1}^t\normsq{\bw_\tau-\bw_{\tau+1}}
\eals
\end{lemma}

\begin{proof}[Proof of Lemma~\ref{lem:OGD+_Guarantees}]
\label{proof:OGD+_Guarantees}
The update rule $\bw_{t+1}=\proj{\bw_t-\eta\bq_t}$ can be re-written as a convex optimization problem over $\K$:
\als
\bw_{t+1}=&\proj{\bw_t-\eta\bq_t}=\underset{\bx\in\K}{\arg\min}\left\{\normsq{\bw_t-\eta\bq_t-\bx}\right\} \\
=&\underset{\bx\in\K}{\arg\min}\left\{\dotprod{\bq_t}{\bx-\bw_t}+\frac{1}{2\eta}\normsq{\bx-\bw_t}\right\}
\eals
The first equality is our update definition, the second is by the definition of the projection operator, and then we rewrite it in a way that does not affect the minimum point.

Now, since $\bw_{t+1}$ is the minimal point of the above convex problem, then from optimality conditions we obtain:
\als
\dotprod{\bq_t+\frac{1}{\eta}(\bw_{t+1}-\bw_t)}{\bx-\bw_{t+1}}\geq0,\quad\forall \bx\in\K
\eals
Re-arranging the above, we get that:
\als
&\dotprod{\bq_t}{\bw_{t+1}-\bx}\leq\frac{1}{\eta}\dotprod{\bw_t-\bw_{t+1}}{\bw_{t+1}-\bx} \\
=&\frac{1}{2\eta}\normsq{\bw_t-\bx}-\frac{1}{2\eta}\normsq{\bw_{t+1}-\bx}-\frac{1}{2\eta}\normsq{\bw_t-\bw_{t+1}}
\eals
Where the equality is an algebraic manipulation.
After summing over $t$ we get: 
\als
&\sum_{\tau=1}^t\dotprod{\bq_\tau}{\bw_{\tau+1}-\bx}\leq\frac{1}{2\eta}\sum_{\tau=1}^t\left(\normsq{\bw_\tau-\bx}-\normsq{\bw_{\tau+1}-\bx}-\normsq{\bw_\tau-\bw_{\tau+1}}\right) \\
&=\frac{\normsq{\bw_1-\bx}-\normsq{\bw_{t+1}-\bx}}{2\eta}-\frac{1}{2\eta}\sum_{\tau=1}^t\normsq{\bw_\tau-\bw_{\tau+1}}\leq\frac{D^2}{2\eta}-\frac{1}{2\eta}\sum_{\tau=1}^t\normsq{\bw_\tau-\bw_{\tau+1}}
\eals
Where the second line is due to splitting the sum into two sums, and using the fact that the first one is a telescopic sum, and lastly, we use the diameter of $\K$.
This establishes the lemma.
\end{proof}

These two lemmas will help us bound the excess loss for an Anytime algorithm.

\begin{lemma}\label{lem:reg_at}
Let $f:\K\to\reals$ be a convex and smooth function with smoothness parameter $L>0$, and define the stochastic error as $\beps_t:=\bq_t-\alpha_t\df{\bx_t}$.
Consider the Anytime-GD update rule $\bw_{t+1}:=\proj{\bw_t-\eta\bq_t},\bx_{t+1}=\bx_t+\frac{\alpha_{t+1}}{\alpha_{1:t+1}}(\bw_{t+1}-\bx_t)$ with $\{\alpha_t=t\}_{t=1}^T$ and learning rate $\eta\leq\frac{1}{4LT}$, then:
\als
\ExpB{\f{\bx_T}}-f^\star\leq\frac{2D^2}{\eta T^2}+\frac{4DG^\star}{T}+\frac{4D}{T^2}\sum_{t=1}^T\Exp{\norm{\beps_t}}
\eals
Where $\bxs:=\underset{\bx\in\K}{\arg\min}\f{\bx}$, and $G^\star:=\norm{\df{\bxs}}$.
\end{lemma}

\begin{proof}[Proof of Lemma~\ref{lem:reg_at}]
Define the expected excess loss at time $t$ as
$\Delta_t := \ExpB{\f{\bx_t}} - f^\star$.
Following the analysis framework of
\citet{reshef2024private,reshef2025privacy}, we begin by deriving an upper bound on $\Delta_t$:
\als
&\alpha_{1:t}\Delta_t:=\alpha_{1:t}\left(\ExpB{\f{\bx_t}}-f^\star\right)\leq\sum_{\tau=1}^t\alpha_\tau\Exp{\dotprod{\df{\bx_\tau}}{\bw_\tau-\bxs}} \\
=&\sum_{\tau=1}^t\alpha_\tau\Exp{\dotprod{\df{\bx_\tau}}{\bw_{\tau+1}-\bxs}}+\sum_{\tau=1}^t\alpha_\tau\Exp{\dotprod{\df{\bx_\tau}}{\bw_\tau-\bw_{\tau+1}}} \\
=&\sum_{\tau=1}^t\Exp{\dotprod{\bq_\tau}{\bw_{\tau+1}-\bxs}}+\sum_{\tau=1}^t\Exp{\dotprod{\beps_\tau}{\bxs-\bw_{\tau+1}}}+\sum_{\tau=1}^t\alpha_\tau\Exp{\dotprod{\df{\bx_\tau}}{\bw_\tau-\bw_{\tau+1}}} \\
\leq&\frac{D^2}{2\eta}-\frac{1}{2\eta}\sum_{\tau=1}^t\Exp{\normsq{\bw_\tau-\bw_{\tau+1}}}+\sum_{\tau=1}^t\Exp{\dotprod{\beps_\tau}{\bxs-\bw_{\tau+1}}}+\sum_{\tau=1}^t\alpha_\tau\Exp{\dotprod{\df{\bx_\tau}}{\bw_\tau-\bw_{\tau+1}}} \\
\leq&\frac{D^2}{2\eta}+\sum_{\tau=1}^t\left(\alpha_\tau\Exp{\dotprod{\df{\bx_\tau}-\df{\bxs}}{\bw_\tau-\bw_{\tau+1}}-\frac{1}{2\eta}\normsq{\bw_\tau-\bw_{\tau+1}}}\right) \\
+&\sum_{\tau=1}^t\Exp{\dotprod{\alpha_\tau\df{\bxs}}{\bw_\tau-\bw_{\tau+1}}}+\sum_{\tau=1}^t\Exp{\dotprod{\beps_\tau}{\bxs-\bw_{\tau+1}}}
\eals

The first inequality is~\Cref{thm:anytime}, then we add and subtract $w_{\tau+1}$ and split it into two sums, split the first sum using $\bq_t=\alpha_t\df{\bx_t}+\beps_t$, bound the first sum using Lemma~\ref{lem:OGD+_Guarantees}, and finally we add and subtract $\df{\bxs}$.

The first sum can be bounded as such:
\als
&\alpha_\tau\Exp{\dotprod{\df{\bx_\tau}-\df{\bxs}}{\bw_\tau-\bw_{\tau+1}}}-\frac{1}{2\eta}\Exp{\normsq{\bw_\tau-\bw_{\tau+1}}} \\
\leq&\frac{\eta}{2}\alpha_\tau^2\Exp{\normsq{\df{\bx_\tau}-\df{\bxs}}}\leq2L\eta\alpha_{1:\tau}\ExpB{\f{\bx_\tau}-f^\star}=2L\eta\alpha_{1:\tau}\Delta_\tau
\eals
Where at first we used Young's inequality, and then we bound $\alpha_\tau^2\leq2\alpha_{1:\tau}$ and use Lemma~\ref{lem:smooth-conv}.

For the second sum, we do this:
\als
&\sum_{\tau=1}^t\Exp{\dotprod{\alpha_\tau\df{\bxs}}{\bw_\tau-\bw_{\tau+1}}}=\sum_{\tau=1}^t(\alpha_\tau-\alpha_{\tau-1})\Exp{\dotprod{\df{\bxs}}{\bw_\tau}}-\alpha_t\Exp{\dotprod{\df{\bxs}}{\bw_{t+1}}} \\
=&\sum_{\tau=1}^t(\alpha_\tau-\alpha_{\tau-1})\Exp{\dotprod{\df{\bxs}}{\bw_\tau-\bw_{t+1}}}\leq\sum_{\tau=1}^t(\alpha_\tau-\alpha_{\tau-1})\norm{\df{\bxs}}\Exp{\norm{\bw_\tau-\bw_{t+1}}} \\
\leq&D\norm{\df{\bxs}}\sum_{\tau=1}^t(\alpha_\tau-\alpha_{\tau-1})=\alpha_t D\norm{\df{\bxs}}=\alpha_tDG^\star
\eals
The first equality is rearrangement of the sum while defining $\alpha_0=0$, then we put the last term into the sum, and use Cauchy-Schwarz, and finally, we use the diameter bound and telescope the sum.

For the third sum we just use Cauchy-Schwartz and the bounded diameter assumption:
\als
\Exp{\dotprod{\beps_\tau}{\bxs-\bw_{\tau+1}}}\leq\ExpB{\norm{\beps_\tau}\cdot\norm{\bxs-\bw_{\tau+1}}}\leq D\Exp{\norm{\beps_\tau}}
\eals
In total, we get:
\als
\alpha_{1:t}\Delta_t\leq2L\eta\sum_{\tau=1}^t\alpha_{1:\tau}\Delta_\tau+\frac{D^2}{2\eta}+\alpha_tDG^\star+D\sum_{\tau=1}^t\Exp{\norm{\beps_\tau}}
\eals
Notice that if $\eta\leq\frac{1}{4LT}$, we get that:
\als
\alpha_{1:t}\Delta_t\leq&\frac{1}{2T}\sum_{\tau=1}^t\alpha_{1:\tau}\Delta_\tau+\frac{D^2}{2\eta}+\alpha_tDG^\star+D\sum_{\tau=1}^t\Exp{\norm{\beps_\tau}} \\
\leq&\frac{1}{2T}\sum_{\tau=1}^T\alpha_{1:\tau}\Delta_\tau+\frac{D^2}{2\eta}+\alpha_TDG^\star+D\sum_{\tau=1}^T\Exp{\norm{\beps_\tau}}
\eals
Averaging over $t$, we get:
\als
\frac{1}{2T}\sum_{\tau=1}^T\alpha_{1:\tau}\Delta_\tau\leq\frac{D^2}{2\eta}+\alpha_TDG^\star+D\sum_{\tau=1}^T\Exp{\norm{\beps_\tau}}
\eals
Inputting this back into the previous bound, we get:
\als
\alpha_{1:T}\Delta_T\leq\frac{D^2}{\eta}+2\alpha_TDG^\star+2D\sum_{\tau=1}^T\Exp{\norm{\beps_\tau}}
\eals
since $\alpha_{1:T}\geq\frac{T^2}{2}$, we get:
\als
\Delta_T\leq\frac{2D^2}{\eta T^2}+\frac{4DG^\star}{T}+\frac{4D}{T^2}\sum_{t=1}^T\Exp{\norm{\beps_t}}
\eals
\end{proof}
Note that this part in the bound is only reliant on the Anytime-GD mechanism.
Now we just need to bound $\beps_t$.
We will define these momentum increments: $\hat{\bs}_t:=\alpha_{t-\tau_t}\bg_{t-\tau_t}-\alpha_{t-\tau_t-1}\tbg_{t-\tau_t-1}$ is the momentum increment given at iteration $t$, $\bs_t:=\alpha_t\bg_t-\alpha_{t-1}\tbg_{t-1}$ is the increment calculated in iteration $t$.
For the first input of each worker, they all use $\bx_1$, so we refer to it as $\bs_1^{(i)}$.
We define $\bbs_t:=\alpha_t\df{\bx_t}-\alpha_{t-1}\df{\bx_{t-1}}$ as the expected momentum at iteration $t$.
We also define $\alpha_0=0$.
$\T(t)$ is the group of missing iterations at time $t$, as in the non-convex analysis.
We define $I(t)$, the set of all workers who had at least 1 iteration at time $t$, and $i(t)$ to be the worker who is acting at iteration $t$.
We can see that $I(t)=\{i(k)\vert k\in\T(t)\}\cup\{i(t)\}$.
The weighted momentum $\bq_t$ is:
\als
\bq_t:=\sum_{k=1}^t\hat{\bs}_k=\sum_{i\in I(t)}\bs_1^{(i)}+\sum_{k\in[t]\setminus\{\T(t)\cup\{1\}\}}\bs_k
\eals
Note that $\sum_{k=1}^t\bbs_k=\alpha_t\df{\bx_t}$.
Thus, the weighted momentum error $\beps_t$ is:
\als
\beps_t:=\bq_t-\alpha_t\df{\bx_t}=\sum_{i\in I(t)}\left(\bs_1^{(i)}-\bbs_1\right)+\sum_{k\in[t]\setminus\{\T(t)\cup{1}\}}\left(\bs_k-\bbs_k\right)-\sum_{k\in\T(t)}\left(\bbs_k-\bbs_1\right)
\eals

\begin{lemma}\label{lem:bound_s}
Let $\bs_t=\alpha_t\bg_t-\alpha_{t-1}\tbg_{t-1}$ and $\bbs_t=\alpha_t\df{\bx_t}-\alpha_{t-1}\df{\bx_{t-1}}$, with assumption~\ref{assume:bound-conv} and the relevant assumption between \ref{assume:indep} or \ref{assume:dep}, and $\alpha_t=t$, then:
\als
\norm{\bs_t}\leq G+2DL &\quad\&\quad \norm{\bbs_t-\bbs_1}\leq3DL \\
\Exp{\normsq{\bs_t-\bbs_t}}\leq\left(\sigma+2D\sigmal\right)^2 &\quad\&\quad \Exp{\normsq{\bs_1^{(i)}-\bbs_1}}\leq\sigma^2
\eals
\end{lemma}

\begin{proof}[Proof of Lemma~\ref{lem:bound_s}]
We will rewrite the momentum increments as $\bs_t=(\alpha_t-\alpha_{t-1})\bg_t+\alpha_{t-1}\left(\bg_t-\tbg_{t-1}\right)$ and $\bbs_t=(\alpha_t-\alpha_{t-1})\df{\bx_t}+\alpha_{t-1}\left(\df{\bx_t}-\df{\bx_{t-1}}\right)$.
Note that:
\als
\norm{\bx_t-\bx_{t-1}}=\frac{\alpha_t}{\alpha_{1:t}}\norm{\bw_t-\bx_{t-1}}\leq\frac{\alpha_t}{\alpha_{1:t}}D
\eals
Since $\alpha_t=t$, we get that $\alpha_t-\alpha_{t-1}=1$ and also:
\als
\alpha_{t-1}\norm{\bx_t-\bx_{t-1}}\leq\frac{\alpha_{t-1}\alpha_t}{\alpha_{1:t}}D=\frac{2(t-1)t}{t(t+1)}D=2\frac{t-1}{t+1}D\leq2D
\eals

For the first part (only relevant with assumption \ref{assume:dep}):
\als
\norm{\bs_t}\leq(\alpha_t-\alpha_{t-1})\norm{\bg_t}+\alpha_{t-1}\norm{\bg_t-\tbg_{t-1}}\leq G+\alpha_{t-1}\norm{\bx_t-\bx_{t-1}}L\leq G+2DL
\eals
Where the first inequality is the triangle inequality, the second inequality is the Lipchitz and smoothness assumptions, and the third inequality is the distance between adjacent queries.

For the second part:
\als
\norm{\bbs_t-\bbs_1}=&\norm{(\alpha_t-\alpha_{t-1})\df{\bx_t}+\alpha_{t-1}(\df{\bx_t}-\df{\bx_{t-1}})-\alpha_1\df{\bx_1}} \\
\leq&\norm{\df{\bx_t}-\df{\bx_1}}+\alpha_{t-1}\norm{\df{\bx_t}-\df{\bx_{t-1}}} \\
\leq&\norm{\bx_t-\bx_1}L+\alpha_{t-1}\norm{\bx_t-\bx_{t-1}}L\leq3DL
\eals
Where the first inequality is the triangle inequality, the second inequality is the smoothness assumption, and the third inequality is the distance between adjacent queries.

For the third part:
\als
&\sqrt{\Exp{\normsq{\bs_t-\bbs_t}}} \\
=&\sqrt{\Exp{\normsq{(\alpha_t-\alpha_{t-1})\left(\bg_t-\df{\bx_t}\right)+\alpha_{t-1}\left(\left(\bg_t-\df{\bx_t}\right)-\left(\tbg_{t-1}-\df{\bx_{t-1}}\right)\right)}}} \\
\leq&\sqrt{\Exp{\normsq{\bg_t-\df{\bx_t}}}}+\alpha_{t-1}\sqrt{\Exp{\normsq{\left(\bg_t-\df{\bx_t}\right)-\left(\tbg_{t-1}-\df{\bx_{t-1}}\right)}}} \\
\leq&\sigma+\alpha_{t-1}\sqrt{\Exp{\normsq{\bx_t-\bx_{t-1}}}}\sigmal\leq\sigma+2D\sigmal
\eals
Where the first inequality is Minkowski’s inequality $\sqrt{\Exp{\normsq{\ba+\bb}}}\leq\sqrt{\Exp{\normsq{\ba}}}+\sqrt{\Exp{\normsq{\bb}}}$, the second inequality is the bounded variance and smoothness variance, and the third inequality is the distance between adjacent queries.

For the fourth part:
\als
\Exp{\normsq{\bs_1^{(i)}-\bbs_1}}=\Exp{\normsq{\df{\bx_1;\bz_1^{(i)}}-\df{\bx_1}}}\leq\sigma^2
\eals
Where we used the bounded variance assumption.
\end{proof}

\paragraph{Data Independent Case.}
We will bound the error in the data independent case.

\begin{lemma}\label{lem:mu2_indep}
Let $f:\K\to\reals$ be a function that satisfies assumptions~\ref{assume:bound-conv} and \ref{assume:indep}.
Then, for any iteration $t\geq1$ with weights $\alpha_t=t$, the weighted error of the delayed momentum in~\Cref{alg:mu2}, $\beps_t:=\alpha_t(\bd_t-\df{\bx_t})$, satisfies:
\als
\Exp{\norm{\beps_t}}\leq(\sigma+2D\sigmal)\sqrt{t}+3DL(M-1)
\eals
\end{lemma}

\begin{proof}[Proof of Lemma~\ref{lem:mu2_indep}]
We can write:
\als
\beps_t=\sum_{i\in I(t)}\left(\bs_1^{(i)}-\bbs_1\right)+\sum_{k\in[t]\setminus\{\T(k)\cup{1}\}}\left(\bs_k-\bbs_k\right)+\sum_{k\in\T(t)}\left(\bbs_k-\bbs_1\right)
\eals
So we will separate using the triangle inequality:
\als
\norm{\beps_t}=\norm{\sum_{i\in I(t)}\left(\bs_1^{(i)}-\bbs_1\right)+\sum_{k\in[t]\setminus\{\T(k)\cup{1}\}}\left(\bs_k-\bbs_k\right)}+\sum_{k\in\T(t)}\norm{\bbs_k-\bbs_1}
\eals
Each term in the last sum can be bounded using Lemma~\ref{lem:bound_s} with $3DL$.
Since there are at most $M-1$, the total is bounded by $3DL(M-1)$.
Since the data is independent of the delay, each $\bs_k-\bbs_k$ is orthogonal:
\als
&\Exp{\normsq{\sum_{i\in I(t)}\left(\bs_1^{(i)}-\bbs_1\right)+\sum_{k\in[t]\setminus\{\T(k)\cup{1}\}}\left(\bs_k-\bbs_k\right)}} \\
=&\sum_{i\in I(t)}\Exp{\normsq{\bs_1^{(i)}-\bbs_1}}+\sum_{k\in[t]\setminus\{\T(k)\cup{1}\}}\Exp{\normsq{{\bs_k-\bbs_k}}}
\eals
The terms in each sum can be bounded using Lemma~\ref{lem:bound_s} to $\sigma^2$, and $(\sigma+2D\sigmal)^2$, since there are $t$ terms in both sums, they are bounded together by $(\sigma+2D\sigmal)^2t$.
In total, we have:
\als
\Exp{\norm{\beps_t}}\leq(\sigma+2D\sigmal)\sqrt{t}+3DL(M-1)
\eals
\end{proof}

\begin{proof}[Proof of~\cref{thm:mu2_indep}]
Continuing from Lemma~\ref{lem:reg_at}, and using Lemma~\ref{lem:mu2_indep}, we get this bound:
\als
\Delta_T\leq&\frac{2D^2}{\eta T^2}+\frac{4DG^\star}{T}+\frac{4D}{T^2}\sum_{t=1}^T\left(\sigma+2D\sigmal\right)\sqrt{t}+\frac{12D^2L(M-1)}{T} \\
\leq&\frac{2D^2}{\eta T^2}+\frac{4DG^\star}{T}+\frac{4D\left(\sigma+2D\sigmal\right)}{\sqrt{T}}+\frac{12D^2L(M-1)}{T}
\eals
Note that our only constraint on the learning rate was $\eta\leq\frac{1}{4LT}$, and as long as we use a learning that satisfy this condition and $\frac{1}{T\eta}\leq\oo{\frac{G^\star}{D}+\left(\frac{\sigma}{D}+\sigmal\right)\sqrt{T}+LM}$, we achieve an excess loss bound of $\oo{\frac{DG^\star}{T}+\frac{D(\sigma+D\sigmal)}{\sqrt{T}}+\frac{D^2LM}{T}}$.
\end{proof}

\paragraph{Data Dependent Case.}
We will bound the error in the data dependent case.

\begin{lemma}\label{lem:bound_eps}
Let $\bd_t$ be as in~\Cref{alg:mu2}, and $\beps_t:=\alpha_t(\bd_t-\df{\bx_t})$, with assumptions~\ref{assume:bound-conv} and \ref{assume:dep} and $\alpha_t=t$, with $M$ workers, then:
\als
\Exp{\norm{\beps_t}}\leq\left(\sigma+2D\sigmal\right)\sqrt{t}+2(G+DL)(M-1)
\eals
\end{lemma}

\begin{proof}[Proof of Lemma~\ref{lem:bound_eps}]
The ideal momentum $\hbq_t$ is:
\als
\hbq_t:=\bs_1^{(1)}+\sum_{k=2}^t\bs_k
\eals
Where we use the first iterate of worker $1$.
The ideal momentum error is $\hbeps_t:=\hbq_t-\alpha_t\df{\bx_t}$.
$\bb_t$ is defined as the difference between the real and ideal momentum:
\als
\bb_t:=\bq_t-\hbq_t=\sum_{i\in I(t)}\bs_1^{(i)}-\bs_1^{(1)}-\sum_{k\in\T(t)}\bs_k
\eals
Note that we write it like this, because there is a possibility that $1\notin I(t)$.
Since $\beps_t-\hbeps_t=\bq_t-\hbq_t=\bb_t$, we can bound $\norm{\beps_t}\leq\norm{\hbeps_t}+\norm{\bb_t}$.
The ideal momentum error $\hbeps_t:=\hbq_t-\alpha_t\df{\bx_t}$ is:
\als
\hbeps_t:=\hbq_t-\alpha_t\df{\bx_t}=\left(\bs_1^{(1)}-\bbs_1\right)+\sum_{k=2}^t\left(\bs_k-\bbs_k\right)
\eals
Note that the ideal momentum $\hbeps_t$ is independent of the delay, thus even with data dependent delays, it will not affect $\hbeps_t$, which we just wrote as a Martingale difference sequence.
Thus:
\als
\Exp{\normsq{\hbeps_t}}=\Exp{\normsq{\bs_1^{(1)}-\bbs_1}}+\sum_{k=2}^t\Exp{\normsq{\bs_k-\bbs_k}}\leq\sigma^2+(\sigma+2D\sigmal)^2(t-1)\leq(\sigma+2D\sigmal)^2t
\eals
Where we used Lemma~\ref{lem:bound_s}.
With Jensen's inequality, we can bound:
\als
\Exp{\norm{\hbeps_t}}\leq\sqrt{\Exp{\normsq{\hbeps_t}}}\leq(\sigma+2D\sigmal)\sqrt{t}
\eals
For the delay, we can use the triangle inequality.
For the case where $1\in I(t)$:
\als
\norm{\bb_t}\leq&\sum_{k\in\T(t)}\norm{\bs_k}+\sum_{i\in I(t)\setminus\{1\}}\norm{\bs_1^{(i)}} \\
\leq&\sum_{k\in\T(t)}(G+2DL)+\sum_{i\in I(t)\setminus\{1\}}G\leq2(G+DL)(M-1)
\eals
Where we used Lemma~\ref{lem:bound_s}, the Lipchitz assumption, and the fact that $|\T(t)|=|I(t)|-1\leq M-1$.
If $1\notin I(t)$, then we get:
\als
\norm{\bb_t}\leq&\sum_{k\in\T(t)}(G+2DL)+\sum_{i\in I(t)\cup\{1\}}G \\
\leq&(G+2DL)(M-2)+GM\leq2(G+DL)(M-1)
\eals
Where we used $|\T(t)|=|I(t)|-1\leq M-2$, since $1\notin I(t)$.
Putting the two bounds together, we got that:
\als
\Exp{\norm{\beps_t}}\leq\left(\sigma+2D\sigmal\right)\sqrt{t}+2(G+DL)(M-1)
\eals
\end{proof}

\begin{proof}[Proof of~\cref{thm:mu2_dep}]
Continuing from Lemma~\ref{lem:reg_at}, and using Lemma~\ref{lem:bound_eps}, we get this bound:
\als
\Delta_T\leq&\frac{2D^2}{\eta T^2}+\frac{4DG^\star}{T}+\frac{4D}{T^2}\sum_{t=1}^T\left(\sigma+2D\sigmal\right)\sqrt{t}+\frac{8D}{T^2}\sum_{t=1}^T(G+DL)(M-1) \\
\leq&\frac{2D^2}{\eta T^2}+\frac{4DG^\star}{T}+\frac{4D\left(\sigma+2D\sigmal\right)}{\sqrt{T}}+\frac{8D(G+DL)(M-1)}{T}
\eals
Note that our only constraint on the learning rate was $\eta\leq\frac{1}{4LT}$, and as long as we use a learning that satisfy this condition and $\frac{1}{T\eta}\leq\oo{\left(\frac{\sigma}{D}+\sigmal\right)\sqrt{T}+\left(\frac{G}{D}+L\right)M}$, we achieve an excess loss bound of $\oo{\frac{D(\sigma+D\sigmal)}{\sqrt{T}}+\frac{D(G+DL)M}{T}}$, since $G^\star\leq G$.
\end{proof}

\clearpage
\section{Experiments}
\label{app:exp}

\subsection{Experimental Setup}
We evaluate the proposed momentum framework on the Fashion-MNIST~\citep{xiao2017fashion} dataset in asynchronous training environments with $M\in\{5,7\}$ workers and $q_1\in\{0.3,0.5,1.0\}$.

\paragraph{Fashion-MNIST Dataset.}
Fashion-MNIST consists of 70,000 grayscale images of clothing items from 10 classes, each with a resolution of $28\times28$ pixels.
The dataset is split into 60,000 training images and 10,000 test images.

\subsection{Model Architecture}
We use a convolutional architecture with two $5\times5$ convolutional layers with 20 and 50 output channels, each followed by a ReLU activation and $2\times2$ max pooling, and two fully connected layers of sizes $800\to50$ and $50\to10$.

\subsection{Baselines}
We compare the proposed Ordered Momentum (\cref{alg:momentum}) and Ordered $\mu^2$-SGD (\cref{alg:mu2}) with constant parameters (\cref{sec:exp}) against the following baselines:
\begin{itemize}
\item Momentum (\cref{eq:naive_asy}, \citet{polyak1964some})
\item $\mu^2$-SGD (\citet{dahanstochastic}, \cref{eq:naive-mu2}).
\item Vanilla SGD (\cref{eq:async-sgd}, \citet{koloskova2022sharper})
\item Delay-Adaptive SGD (\cref{eq:delay-adaptive-sgd}, \citet{mishchenko2022asynchronous}).
\end{itemize}

For all methods, the stepsize is selected from
$\eta\in\{0.001,0.005,0.01,0.05,0.07,0.1,0.2,0.3,0.5\}$.
We use $\beta=0.1$, corresponding to the conventional momentum coefficient $1-\beta=0.9$, and set the query-momentum parameter to $\gamma=0.1$.
For every configuration, we train the model for $20$ epochs using a mini-batch size of $1024$ and report the final test accuracy averaged over three random seeds.

\subsection{Data-Dependent Delay Model}
\label{sec:delay-model}
We design a controlled simulation of data-dependent delays to reflect heterogeneous and input-dependent computation times commonly observed in asynchronous training.

\paragraph{Imbalanced worker arrivals.}
We simulate asynchrony by sampling the identity of the next arriving worker at each server iteration.
Specifically, for $M$ workers indexed by $i\in\{1,\dots,M\}$, worker $i$ is sampled with probability proportional to its index, i.e.,
\als
p_i=\frac{i}{\sum_{j=1}^Mj}=\frac{2i}{M(M+1)}
\eals
The number of server iterations between consecutive arrivals of worker $i$ therefore follows a geometric distribution with parameter $p_i$ and determines that worker's delay.
Workers with larger indices are sampled more frequently and thus have shorter expected delays.
This setup induces heterogeneous delays across workers and captures variability in system-level scheduling.

\paragraph{Data-dependent delay times.}
To model variability in computation time induced by the data itself, we introduce delays that depend explicitly on sample characteristics.

Samples assigned to the \emph{long-delay} group are only made available after a worker-specific delay threshold $\tau_i$, while preserving the overall sampling proportions across groups.
Let $T_i$ denote the waiting time until worker $i$ arrives, which follows a geometric distribution with parameter $p_i$, so that
\als
\prob{T_i>\tau_i}=(1-p_i)^{\tau_i}
\eals
To sample from a prescribed mixture distribution $\D=q_1\D_1+q_2\D_2$, with $q_1+q_2=1$, we select a sample from $\D_1$ if $T_i>\tau_i$ and from $\D_2$ otherwise.
Enforcing $\prob{T_i>\tau_i}=q_1$ yields
\als
\tau_i=\frac{\log(q_1)}{\log(1-p_i)}
\eals
As a result, if worker $i$ arrives after $\tau_i$, we sample $z\sim\D_1$; otherwise, we sample $z\sim\D_2$.
This procedure ensures that long-delay samples are introduced according to the desired mixture proportions while respecting worker-specific arrival patterns.

\subsection{Data-Dependent Delays in the Experiments}
\label{sec:data-dependent-image-experiments}

We instantiate the data-dependent delay model from \cref{sec:delay-model} on Fashion-MNIST by associating longer delays with higher-quality images.
We sweep $q_1\in\{0.3,0.5,1.0\}$ in the mixture
\als
\D=q_1\D_{\mathrm{high}}+(1-q_1)\D_{\mathrm{low}},
\eals
where $\D_{\mathrm{high}}$ is the original training distribution and $\D_{\mathrm{low}}$ is a degraded version of the same distribution.
The degradation modifies only the input image and preserves its label.
For $q_1=1.0$, every sample is drawn from $\D_{\mathrm{high}}$, so image quality is independent of the delay; this is our data-independent setting.

\paragraph{Image-quality distributions.}
Samples from $\D_{\mathrm{high}}$ undergo only the standard dataset-specific normalization.
To construct $\D_{\mathrm{low}}$, we first downsample each image by a factor of two using strided subsampling and then restore it to its original resolution through nearest-neighbor replication.
We subsequently add independent Gaussian noise with standard deviation $0.1$.
Fashion-MNIST images are reduced from $28\times28$ to $14\times14$ before being restored to their original resolution.
Representative high- and low-quality samples are shown in \cref{fig:fashion-quality-comparison}.

\begin{figure}[t]
\centering
\includegraphics[width=\linewidth]{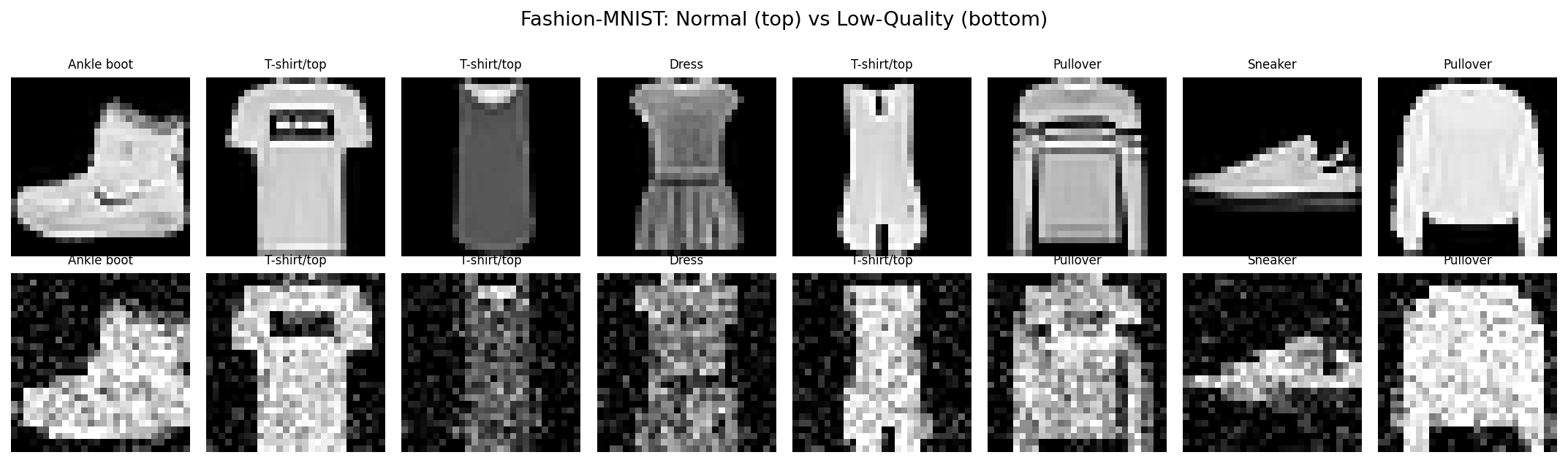}
\caption{Representative high- and low-quality Fashion-MNIST samples.
The top row shows the original images from $\D_{\mathrm{high}}$, and the bottom row shows the corresponding degraded images from $\D_{\mathrm{low}}$.
The degradation consists of downsampling, nearest-neighbor upsampling, and additive Gaussian noise, while the labels remain unchanged.}
\label{fig:fashion-quality-comparison}
\end{figure}

\paragraph{Coupling image quality and delay.}
We identify the long-delay distribution $\D_1$ in \cref{sec:delay-model} with the high-quality distribution and the short-delay distribution $\D_2$ with the low-quality distribution:
\als
\D_1=\D_{\mathrm{high}},
\qquad
\D_2=\D_{\mathrm{low}}.
\eals
Accordingly, when worker $i$ has waiting time $T_i$, its sample is drawn according to
\als
\bz\sim
\begin{cases}
\D_{\mathrm{high}}, & T_i>\tau_i,\\
\D_{\mathrm{low}}, & T_i\leq\tau_i,
\end{cases}
\qquad
\tau_i=\frac{\log(q_1)}{\log(1-p_i)}.
\eals
Since $\prob{T_i>\tau_i}=q_1$, this construction preserves the prescribed marginal fraction of high-quality samples while making image quality statistically dependent on the delay: original images are associated with longer waiting times, whereas degraded images are associated with shorter waiting times.

\begin{figure}[H]
\centering
\includegraphics[width=\linewidth]{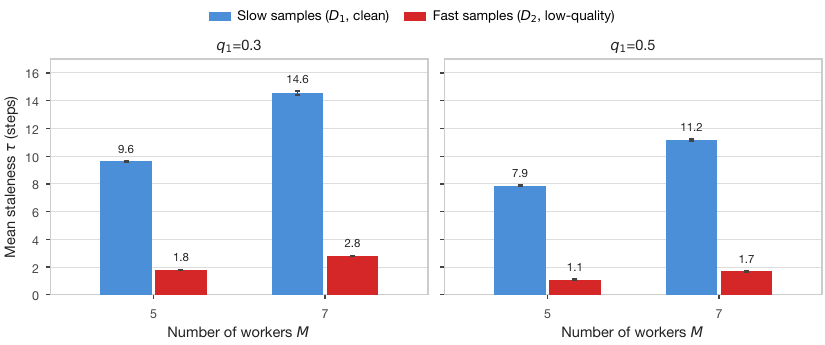}
\caption{\small
\textbf{Empirical coupling between image quality and delay on Fashion-MNIST.}
Mean staleness is reported for clean samples from $\D_1=\D_{\mathrm{high}}$ and degraded samples from $\D_2=\D_{\mathrm{low}}$, for $M\in\{5,7\}$ workers and $q_1\in\{0.3,0.5\}$.
Clean samples have larger mean staleness than degraded samples, confirming that the construction in \cref{sec:delay-model} induces the intended data-dependent delays.}
\label{fig:delay-coupling-fmnist}
\end{figure}

\subsection{Complete Results}

\begin{figure}[H]
\centering
\includegraphics[width=\linewidth]{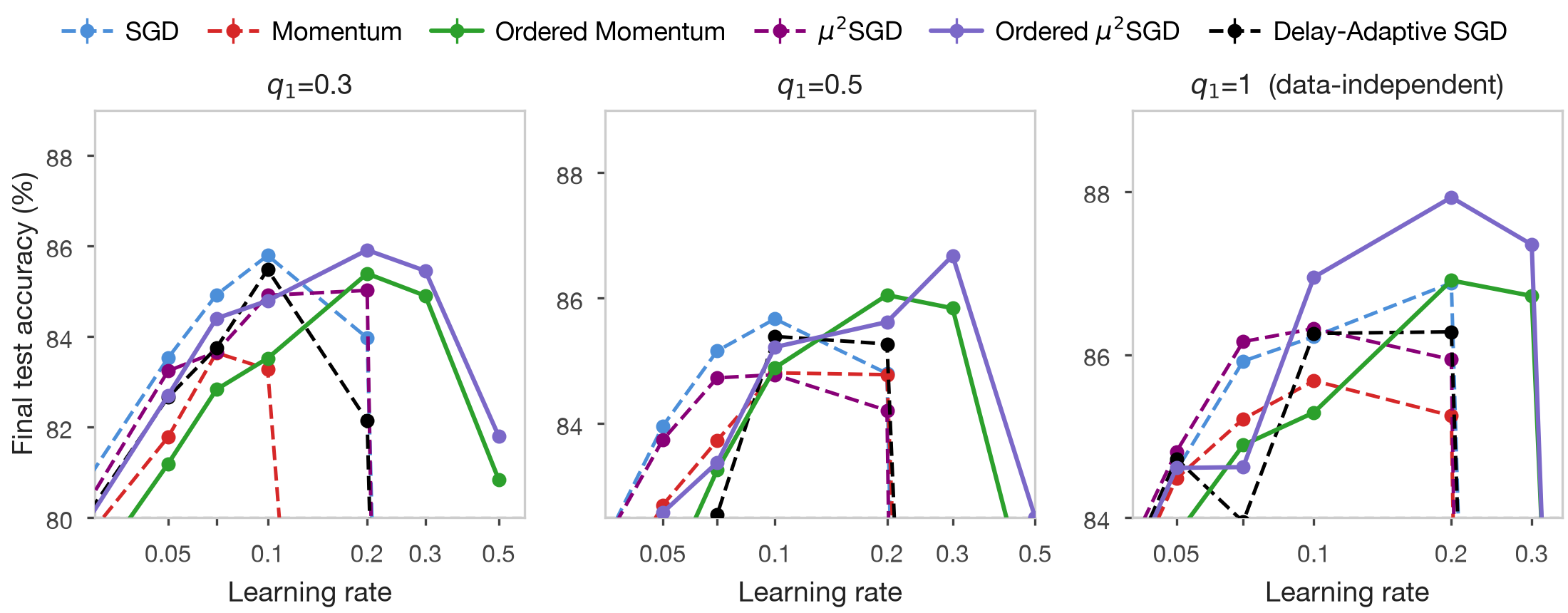}
\caption{\small
\textbf{Learning-rate robustness on Fashion-MNIST with $M=5$ workers (means only).}
We report the same mean final test accuracies as in \cref{fig:lr-robustness-w5}, omitting error bars for clarity.
This view highlights the relative performance of the methods and their stable learning-rate ranges.}
\label{fig:lr-robustness-w5-nostd}
\end{figure}

\clearpage

\begin{figure}[H]
\centering
\includegraphics[width=\linewidth]{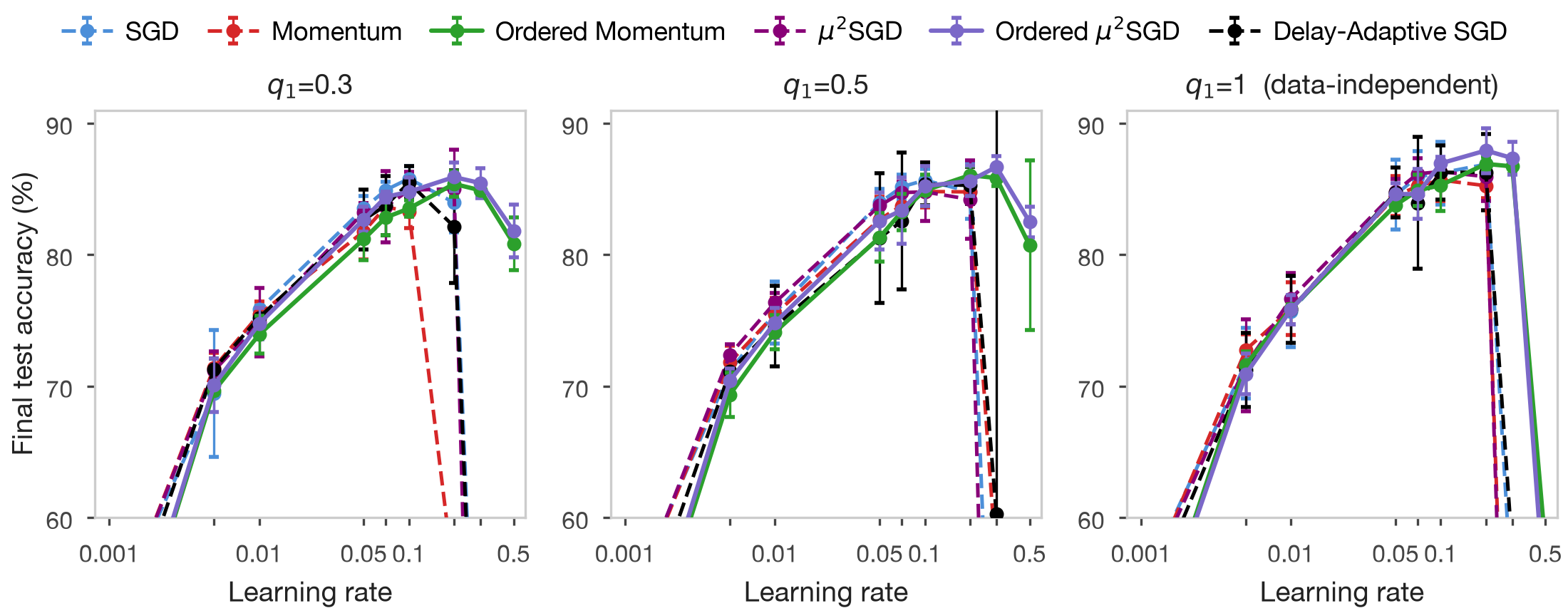}
\caption{\small
\textbf{Learning-rate robustness on Fashion-MNIST with $M=5$ workers and the error bars indicate one standard deviation.}}
\label{fig:lr-robustness-w5}
\end{figure}

\begin{figure}[H]
\centering
\includegraphics[width=\linewidth]{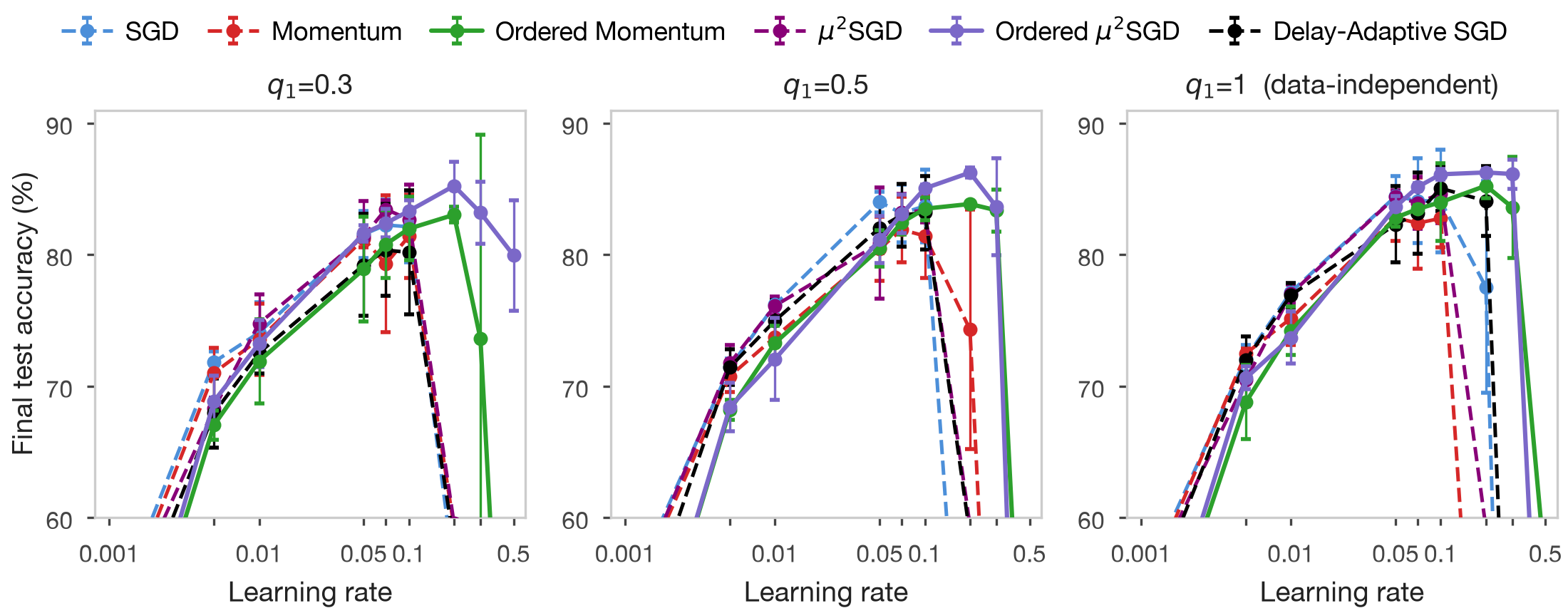}
\caption{\small
\textbf{Learning-rate robustness on Fashion-MNIST with $M=7$ workers and the error bars indicate one standard deviation.}}
\label{fig:lr-robustness-w7-std}
\end{figure}

\end{document}